\documentclass[letterpaper, 10 pt, conference]{cls/ieeeconf}
\IEEEoverridecommandlockouts

\usepackage{graphics} 
\usepackage{graphicx,tabularx,adjustbox}
\usepackage{tikz}
\usetikzlibrary{fadings}
\usepackage{multirow}
\usepackage{hhline}
\usepackage{xcolor}
\usepackage{float}
\usepackage{subcaption}
\usepackage{amsmath} 
\usepackage{amssymb}  
\usepackage{amsfonts,amstext,dsfont,mathtools,bbm}
 
\usepackage{amsthm}
\usepackage{balance}
\usepackage{cite}
\usepackage{xspace}

\usepackage[bookmarks=true, colorlinks, breaklinks=true]{hyperref}
\usepackage{xurl} 

\DeclareMathOperator*{\diag}{diag}

\newcommand{\bfe}{\mathbf{e}}
\newcommand{\bff}{\mathbf{f}}
\newcommand{\bfg}{\mathbf{g}}

\newcommand{\bfo}{\mathbf{o}}

\newcommand{\bfq}{\mathbf{q}}

\newcommand{\bfx}{\mathbf{x}}
\newcommand{\bfy}{\mathbf{y}}

\newcommand{\bfzeta}{\boldsymbol{\zeta}}

\newcommand{\bfxi}{\boldsymbol{\xi}}

\newcommand{\bbR}{\mathbb{R}}

\newcommand{\calL}{\mathcal{L}}

\newcommand{\calO}{\mathcal{O}}
\newcommand{\calP}{\mathcal{P}}

\newcommand{\calT}{\mathcal{T}}
\newcommand{\calU}{\mathcal{U}}
\newcommand{\calV}{\mathcal{V}}

\newtheorem{proposition}{Proposition}

\def\htwomapping{$\text{H}_2$-Mapping\xspace}

\def\methodname{OREN\xspace}
\def\papertitle{OREN: Octree Residual Network for Real-Time Euclidean Signed Distance Mapping}

\title{\Large \bf\papertitle}

\author{\authorblockN{Zhirui Dai$^{1}$$^{*}$ \qquad Qihao Qian$^{1}$$^{*}$ \qquad Tianxing Fan$^{1}$ \qquad Nikolay Atanasov$^{1}$}
\thanks{$^{*}$Equal contribution}%
\thanks{$^{1}$The authors are with the Department of Electrical and Computer Engineering, University of California San Diego, La Jolla, CA 92093, USA, e-mails: {\tt\small \{zhdai,q2qian,t2fan,natanasov\}@ucsd.edu}.}%
\thanks{We gratefully acknowledge support from ARL DCIST CRA W911NF-17-2-0181 and the Ministry of Trade, Industry and Energy (MOTIE), Korea, under the Strategic Technology Development Program, supervised by Korea Institute for Advancement of Technology (KIAT) [Grant No. P0026052].}%
\thanks{Code available at \url{https://github.com/ExistentialRobotics/oren}}%
}

\begin{document}
\bstctlcite{IEEEexample:BSTcontrol}

\maketitle
\thispagestyle{empty}
\pagestyle{empty}

\begin{abstract}
Reconstructing signed distance functions (SDFs) from point cloud data benefits many robot autonomy capabilities, including localization, mapping, motion planning, and control. Methods that support online and large-scale SDF reconstruction often rely on discrete volumetric data structures, which affects the continuity and differentiability of the SDF estimates. Neural network methods have demonstrated high-fidelity differentiable SDF reconstruction but they tend to be less efficient, experience catastrophic forgetting and memory limitations in large environments, and are often restricted to truncated SDF. This work proposes \methodname, a hybrid method that combines an explicit prior from octree interpolation with an implicit residual from neural network regression. Our method achieves non-truncated (Euclidean) SDF reconstruction with computational and memory efficiency comparable to volumetric methods and differentiability and accuracy comparable to neural network methods. Extensive experiments demonstrate that \methodname outperforms the state of the art in terms of accuracy and efficiency, providing a scalable solution for downstream tasks in robotics and computer vision.
\end{abstract}


\section{Introduction}
\label{sec:introduction}

Accurate and differentiable 3D geometric reconstruction is critical in many robot autonomy and computer vision settings, including localization and mapping \cite{ortiz_isdf_2022,pan_pin-slam_2024}, rendering and AR/VR \cite{Chou_2023_ICCV,wang_neus2_2023}, autonomous robot navigation \cite{oleynikova_voxblox_2017,long_sensor-based_2025} and robot manipulation \cite{ReDSDF,li2024config}. In robotics, fast updates of the environment model from sensor observations and access to gradient information are important for safe navigation and precise environment interaction, while low memory footprint is important for scalability and onboard operation.

In this work, we focus on signed distance function (SDF) reconstruction. Given a query point, an SDF returns the signed distance to the nearest surface in the environment with sign indicating whether the query is in free (positive) or occupied (negative) space. SDFs have received increasing attention due to their constant-time complexity for distance and collision queries and their ability to capture complex obstacle surfaces implicitly as a zero-level set.

SDF reconstruction methods fall into roughly three categories: volumetric methods (e.g., \cite{newcombe_kinectfusion_2011,oleynikova_voxblox_2017}), Gaussian Process (GP) methods (e.g., \cite{lee2019gpis,wu_vdb-gpdf_2025}), and neural network methods (e.g., \cite{park_deepsdf_2019,wang_neus_2021}). We review representative works in Sec.~\ref{sec:related_work}. Volumetric methods utilize advanced data structures, like octrees and hashmaps, and are known for their real-time performance and scalability to large scenes. However, they provide non-differentiable SDF estimates and require growing storage for higher accuracy. GP methods learn continuous SDF models with uncertainty quantification but often suffer from high computational complexity and poor scalability. Recently, neural network methods have shown great potential to learn compact and accurate SDF representations but are often restricted to truncated SDF and struggle with catastrophic forgetting in large scenes or in online settings.

\tikzfading[name=fade out, inner color=transparent!0, outer color=transparent!100]

\begin{figure}[t]
    \centering
    \includegraphics[width=\linewidth,trim={0 50pt 0 50pt},clip]
    {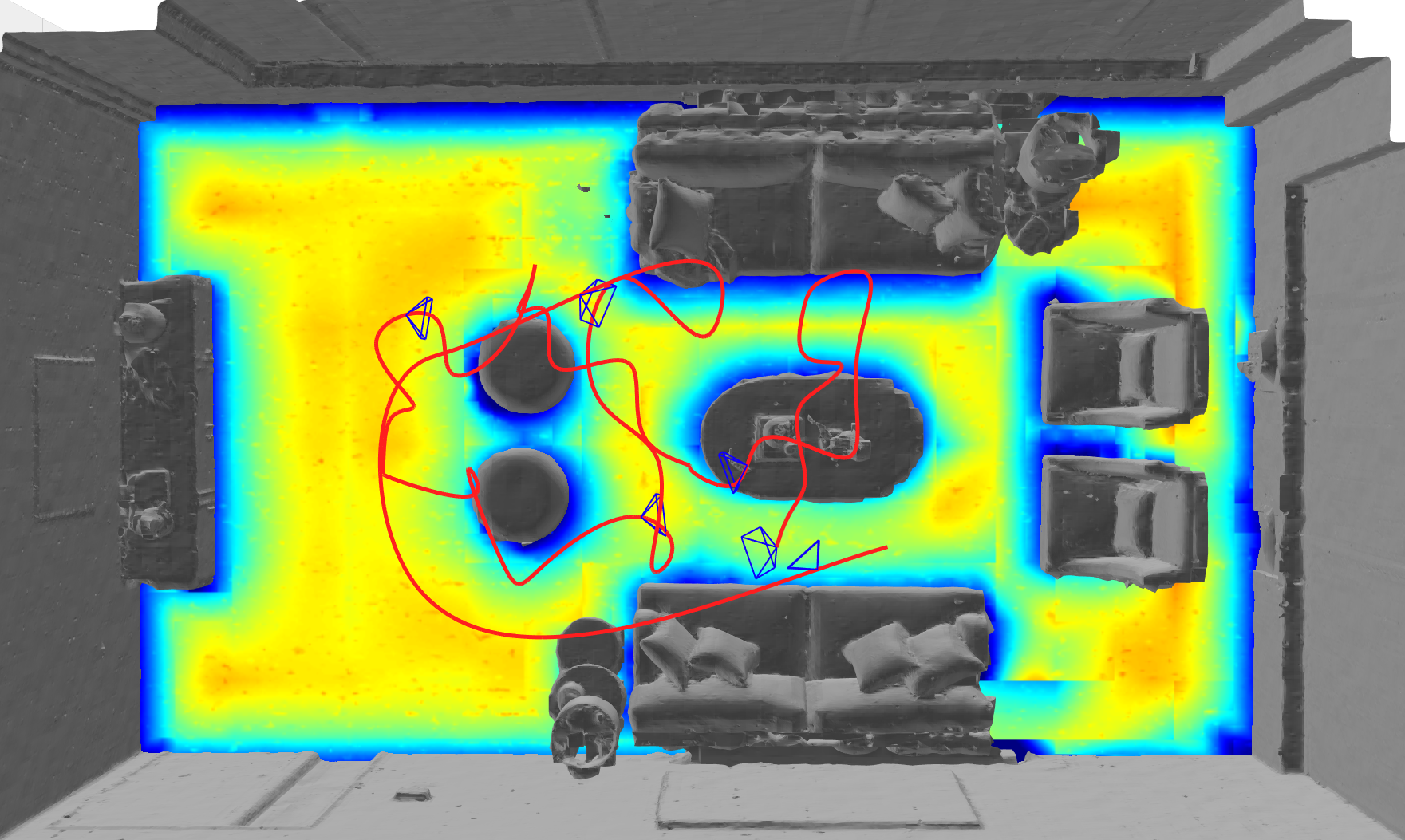}
    \caption{\small \methodname reconstructs an accurate Euclidean signed distance function online from streaming point cloud data.}
    \vspace{-1em}
    \label{fig:teaser}
\end{figure}

We propose \methodname, a hybrid method that combines the strengths of volumetric methods, in the form of an explicit SDF prior obtained from octree interpolation, and neural network methods, in the form of implicit features decoded into a residual correction of the prior. To construct the prior, we use a semi-sparse octree with SDF and gradient estimates stored at the octant vertices and gradient-augmented interpolation to obtain SDF priors at arbitrary query positions. We augment the prior prediction with a neural network residual, which recovers fine geometric details of the observed surface from implicit features. We train our hybrid explicit-implicit model using three loss functions that supervise both near-surface and distant SDF values and accelerate the convergence to achieve real-time accurate Euclidean SDF reconstruction.

The closest works to ours are \htwomapping \cite{jiang_h2-mapping_2023} and Gradient-SDF \cite{sommersang2022}.
Like ours, \htwomapping applies trilinear interpolation in a sparse octree to obtain an SDF prior and trains a neural network residual. In contrast, \htwomapping reconstructs only truncated (near-surface) SDF.
Gradient-SDF is voxel-based and also uses SDF gradients but only for improving surface reconstruction and photometric bundle adjustment, ignoring free-space SDF estimation.
In contrast, our method optimizes the octree parameters to obtain an accurate Euclidean SDF prior, corrects the prior via neural residual predictions and achieves real-time operation.

In summary, our work makes the following contributions.
\begin{itemize}
  \item We introduce a new gradient-augmented interpolation in a semi-sparse octree to obtain an SDF prior, improving the accuracy, memory, and training speed for subsequent SDF residual learning.

  \item We formulate a hybrid model that combines the explicit priors with an implicit neural residual, enabling accurate SDF learning both near to and far from the observed surfaces. We also design loss functions to encourage globally accurate SDF learning and accelerate the training process to achieve real-time performance.
\end{itemize}

\begin{figure*}[t]
  \centering
  \includegraphics[width=0.95\linewidth, trim={0 10pt 0 0}, clip]{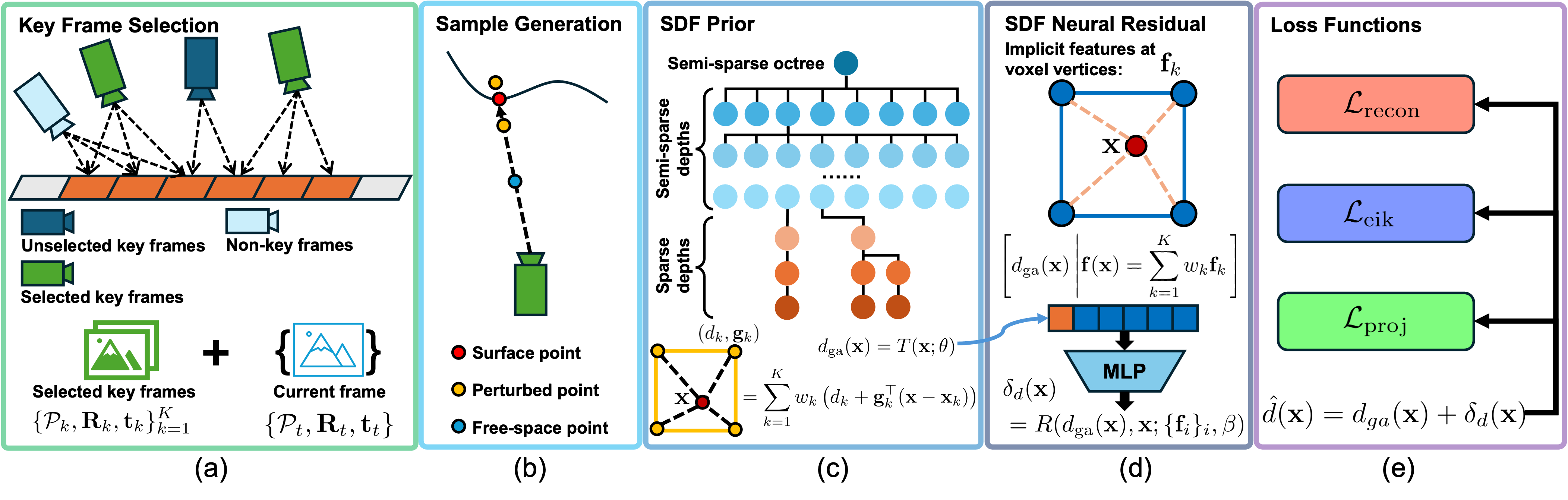}
  \caption{\small Method overview: a) We keep key frames with small overlap that maximize surface coverage; b) From the selected key frames and the current frame, we generate three sample types: \textcolor{red}{surface} points, \textcolor{orange}{perturbed} points around the surface, and \textcolor{cyan}{free-space} points; c) To predict SDF, we first obtain a prior $d_{ga}(\bfx)$ via gradient-augmented interpolation in a semi-sparse octree, where each vertex stores an SDF value and gradient; d) We also interpolate an implicit neural feature for $\bfx$ from features stored at the octant vertices, and feed it with the prior $d_\text{ga}(\bfx)$ into an MLP decoder to obtain an SDF residual $\delta_d(\bfx)$; e) The prior and residual are combined into the final SDF prediction $\hat{d}(\bfx) = d_{ga}(\bfx) + \delta_d(\bfx)$, and the model is trained with three loss functions: \textcolor{red}{reconstruction} loss, \textcolor{violet}{Eikonal} loss and \textcolor{green}{projection} loss.}
  \label{fig:method_overview}
  \vspace{-1em}
\end{figure*}

\section{Related Work}
\label{sec:related_work}

Various methods have been proposed to learn SDF from point cloud data. They can be categorized into three groups: volumetric methods \cite{curless_volumetric_1996,newcombe_kinectfusion_2011,kahler_very_2015,oleynikova_voxblox_2017,han_fiesta_2019,pan_voxfield_2022,millane_nvblox_2024}, GP-based methods \cite{lee2019gpis,lan2021loggpis,wu_vdb-gpdf_2025}, and neural network based methods \cite{park_deepsdf_2019,ortiz_isdf_2022}. We review these categories and then discuss the recent trend of using hybrid models for SDF reconstruction.

\subsection{Volumetric SDF Reconstruction}

Volumetric methods \cite{oleynikova_voxblox_2017,pan_voxfield_2022} achieve real-time truncated SDF (TSDF) reconstruction using voxel hashing for efficient updates and queries. Voxblox \cite{oleynikova_voxblox_2017} builds a TSDF layer by projective distance (from voxel center to observed surface point) and updates values via breadth-first search (BFS). Both the projective distance and the BFS path length introduce inaccuracies in SDF estimates.
Subsequent works \cite{han_fiesta_2019,pan_voxfield_2022} reduce these errors by using non-projective distance and replacing BFS path length with the distance to the nearest oriented surface point.
Gradient-SDF \cite{sommersang2022} shows that estimating TSDF together with its gradient improves accuracy and can be used for SDF-based camera tracking and bundle adjustment.
However, these methods rely on discrete SDF representations, which are non-differentiable and have accuracy limited to the grid resolution.
In contrast, our method enables real-time learning of continuous and differentiable SDF. We first estimate SDF with explicit discrete priors stored in an octree, and then use implicit neural features to predict residuals and correct the prior, which forms a compact differentiable representation of continuous SDF.

\subsection{Gaussian Process SDF Reconstruction}

GP methods learn continuous SDF representations, and support gradient computation and uncertainty quantification \cite{lee2019gpis,lan2021loggpis,wu_vdb-gpdf_2025}. GPIS \cite{lee2019gpis} uses GP to estimate oriented surface points and regress SDF online, learning the surface implicitly as the zero-level set of SDF. However, GPIS cannot extrapolate SDF predictions away from the surface. Log-GPIS \cite{lan2021loggpis} and VDB-GPDF \cite{wu_vdb-gpdf_2025} learn unsigned distance function in log space, which generalizes well globally but does not provide sign information. GP methods scale poorly to large environments due to the cubic complexity of the matrix inverse during training.
In comparison, our method uses octree interpolation, which has $O(1)$ complexity, and matrix multiplication, which has roughly quadratic complexity $O(n^2)$ for $n$ hidden dimension, to compute the SDF prior and the SDF residual respectively. Our octree structure also has a smaller memory footprint than GP's gram matrix.

\subsection{Neural Network SDF Reconstruction}

DeepSDF \cite{park_deepsdf_2019} was among the first to demonstrate that neural networks can learn compact and continuous implicit SDF representations. This inspired many subsequent neural SDF methods.
iSDF \cite{ortiz_isdf_2022} incrementally learns SDF by iteratively updating an MLP with training samples from a key frame set. iSDF also uses Eikonal regularization to enforce the Eikonal property of SDF.
NeuS \cite{wang_neus_2021} jointly learns SDF and a neural radiance field (NeRF) \cite{nerf2020}, letting the two improve each other via SDF-based volume density estimates.
Other works, like IGR \cite{gropp_implicit_2020} and NGLoD \cite{takikawa_lod_2021}, propose various network designs, loss functions, and training procedures to improve SDF reconstruction accuracy.
These works show that neural networks can learn near-surface SDF accurately enough for high-fidelity surface reconstruction. However, accuracy far from the surface is rarely evaluated.
Neural network methods \cite{park_deepsdf_2019,gropp_implicit_2020,takikawa_lod_2021,wang_hotspot_2024} that learn Euclidean SDF are often object-level and require extensive training data and time for satisfactory accuracy.

\subsection{Hybrid Methods for SDF Reconstruction}

Recently, hybrid models combining explicit geometric structures with implicit neural features show promising results. PIN-SLAM \cite{pan_pin-slam_2024} stores neural features in near-surface voxels. Given a query, its SDF prediction is a weighted sum of $k$ predictions, obtained by feeding the $k$ nearest neural features and local positions into a global decoder.
\htwomapping \cite{jiang_h2-mapping_2023} combines an octree-based SDF prior with a neural residual.
LCP-Fusion \cite{lcp-fusion-2024} extends \htwomapping to SLAM, combining the SDF prior and implicit neural feature for localization.
However, these methods still learn truncated SDF.
HIO-SDF \cite{hio-sdf_2024} removes truncation by running Voxfield \cite{pan_voxfield_2022} first to generate global SDF priors, which are combined with local SDF approximations to train a neural network. However, the accuracy and speed are limited by the volumetric method. As the observed area grows, the fixed-size network tends to learn an over-smooth SDF.

In contrast, our method, \methodname, builds a semi-sparse octree to store the prior of SDF values and gradients, which is extendable as the scene grows and efficiently represents the SDF of the whole space. With gradient-augmented interpolation in the octree, our method can produce more accurate SDF priors, leaving more capacity for the subsequent network to recover surface geometric details.

\section{Problem Statement}
\label{sec:problem_statement}

Consider a 3D environment with a set of obstacles $\calO \subset \bbR^3$. The SDF $d: \bbR^3 \rightarrow \bbR$ of $\calO$ is defined as the shortest distance from any point $\bfx \in \bbR^3$ to the obstacle surface $\partial \calO$, with a sign indicating whether $\bfx$ is inside or outside of $\calO$:
\begin{equation} \label{eq:sdf_definition}
  d(\bfx) =
  \begin{cases}
    \phantom{+}\min_{\bfy \in \partial \calO} \left\|\bfx - \bfy\right\|_2, & \bfx \not\in \calO, \\
    -\min_{\bfy \in \partial \calO} \left\|\bfx - \bfy\right\|_2, & \bfx \in \calO.
  \end{cases}
\end{equation}
The SDF satisfies two key properties: 1) the obstacle surface is encoded as the zero-level set, $d(\bfx) = 0$, $\forall \bfx \in \partial\calO$, and 2) the gradient of $d(\bfx)$ is the unit vector pointing away from the nearest surface point and satisfies an Eikonal equation~\cite{ortiz_isdf_2022}:
\begin{equation}
\nabla d(\bfx)=\frac{\bfx-\bfx_*}{d(\bfx)}, \quad \left\|\nabla d(\bfx)\right\|_2=1, \; \text{a.e.},\label{eq:sdf_constraints}
\end{equation}
where $\bfx_* \in \arg \min_{\bfy\in\partial\calO} \left\|\bfx-\bfy\right\|_2$.

Given a stream of point clouds from a range sensor (e.g., LiDAR or depth camera),
$\{\bfo_t, \calP_t\}_{t=1}$,
where $\bfo_t$ is the sensor position at time $t$ and $\calP_t$ is the set of observed surface points in the global frame, our objective is to approximate the SDF $d(\bfx)$ of $\calO$ as a scalar field, $\hat{d}: \bbR^3 \rightarrow \bbR$. We also want $\hat{d}$ to capture the SDF gradient accurately.
\section{Octree Residual Network for Learning SDF}
\label{sec:technical_approach}

Our method employs a hybrid model to reconstruct SDF. We use a semi-sparse octree, where certain octants contain no surface, to store explicit SDF and gradient estimates and compute a coarse SDF prior, described in Sec. \ref{sec:sdf_prior}. To recover geometric details, we also store implicit neural features at the octant vertices and use an MLP decoder to correct the prior SDF estimates.
The implicit neural residual prediction is described in Sec. \ref{sec:sdf_residual}. To train our model efficiently, we maintain key frames that cover the observed surface with small overlap between adjacent frames, as described in Sec. \ref{sec:key_frame_selection}. Then, from the key frames and the latest frame, we generate a dataset of different sample types, discussed in Sec.~\ref{sec:dataset_generation}, and use it to train the model with loss functions proposed in Sec. \ref{sec:loss_functions}. Fig~\ref{fig:method_overview} shows an overview.

\subsection{SDF Prior via Gradient-Augmented Octree Interpolation}
\label{sec:sdf_prior}

\subsubsection{Semi-Sparse Octree}

The SDF prior at position $\bfx$ is obtained by interpolation in a semi-sparse octree with $N$ layers. We call an octree layer \emph{dense} when all octants form a regular grid; \emph{semi-sparse} when child octants containing surface points and all their siblings are created; and \emph{sparse} when only child octants containing surface points are created.

We use a semi-sparse octree of resolution $r$, where the first $M$ layers are semi-sparse and the remaining $N-M$ layers are sparse. This is illustrated in Fig.~\hyperref[fig:method_overview]{2c}. In each vertex $\bfx_k$ of an octant with $k \in \{1,\ldots,8\}$, we store estimates $d_k \in \bbR$ and  $\bfg_k\in\bbR^3$ of the SDF $d(\bfx_k)$ and its gradient $\nabla d(\bfx_k)$, respectively, which are learnable during training. To maintain memory efficiency, a vertex is shared across neighboring octants from different tree depths. For example, the eight vertices of an octant are also included in the vertices of its eight child octants. This semi-sparse structure is essential to obtain a good SDF prior, especially for query positions away from the surface. The on-demand initialization of all child octants in the first $M$ layers costs extra memory but improves the accuracy significantly.
The choice of $M$ trades memory for accuracy: more semi-sparse layers yield a more accurate prior but use more memory. Since the octant count grows roughly eight-fold per layer, making the finest layers semi-sparse adds substantial memory but little accuracy, as the prior is already accurate near the surface. We therefore keep only the first $M$ layers semi-sparse and the rest sparse.

\begin{figure}[t]
    \centering
    \begin{subfigure}[t]{0.24\linewidth}
      \includegraphics[width=\linewidth]{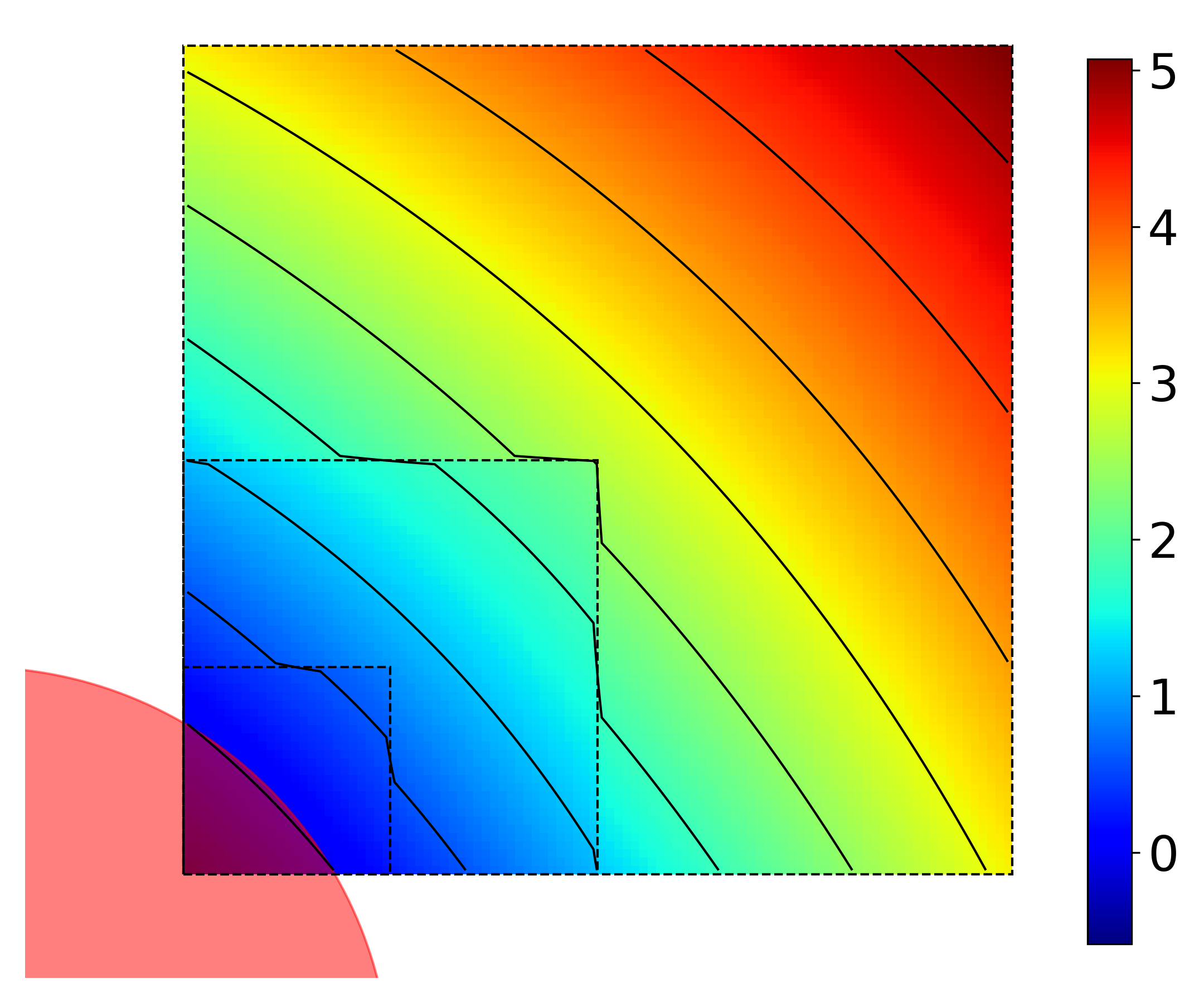}
      \caption{}
      \label{fig:comp_octree_structure_a}
    \end{subfigure}%
    \begin{subfigure}[t]{0.24\linewidth}
      \includegraphics[width=\linewidth]{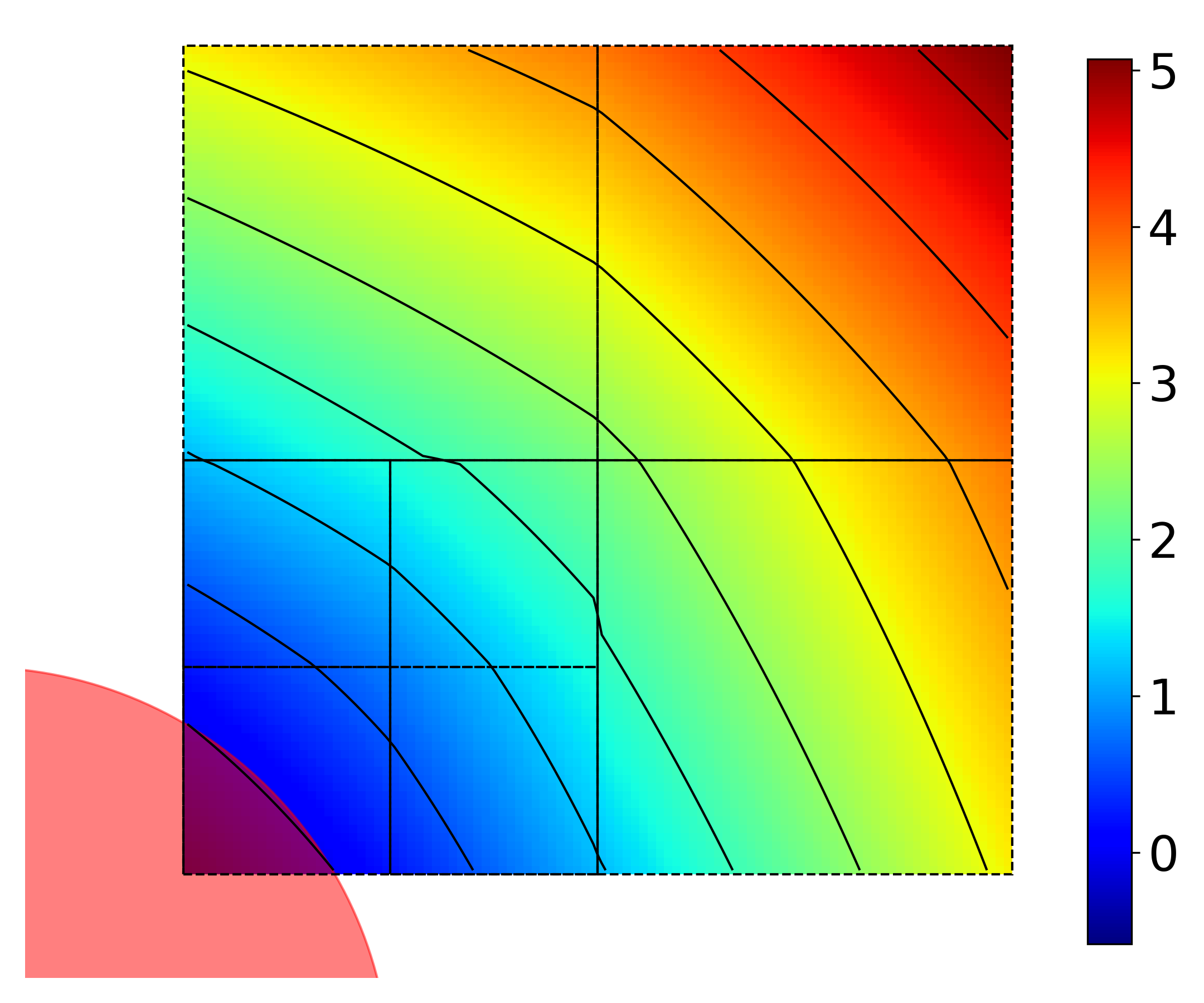}
      \caption{}
      \label{fig:comp_octree_structure_b}
    \end{subfigure}
    \begin{subfigure}[t]{0.24\linewidth}
      \includegraphics[width=\linewidth]{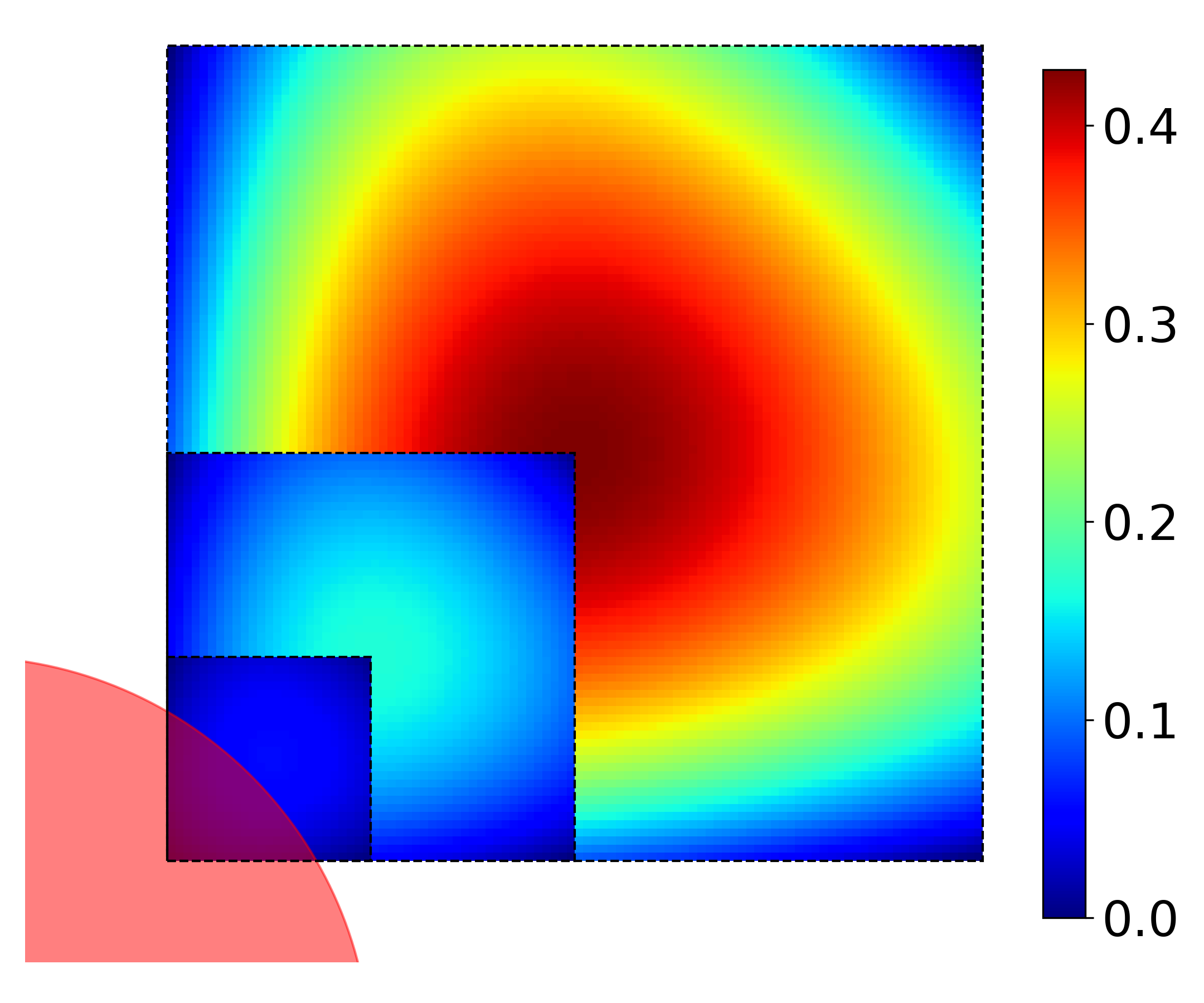}
      \caption{}
      \label{fig:comp_octree_structure_c}
    \end{subfigure}%
    \begin{subfigure}[t]{0.24\linewidth}
      \includegraphics[width=\linewidth]{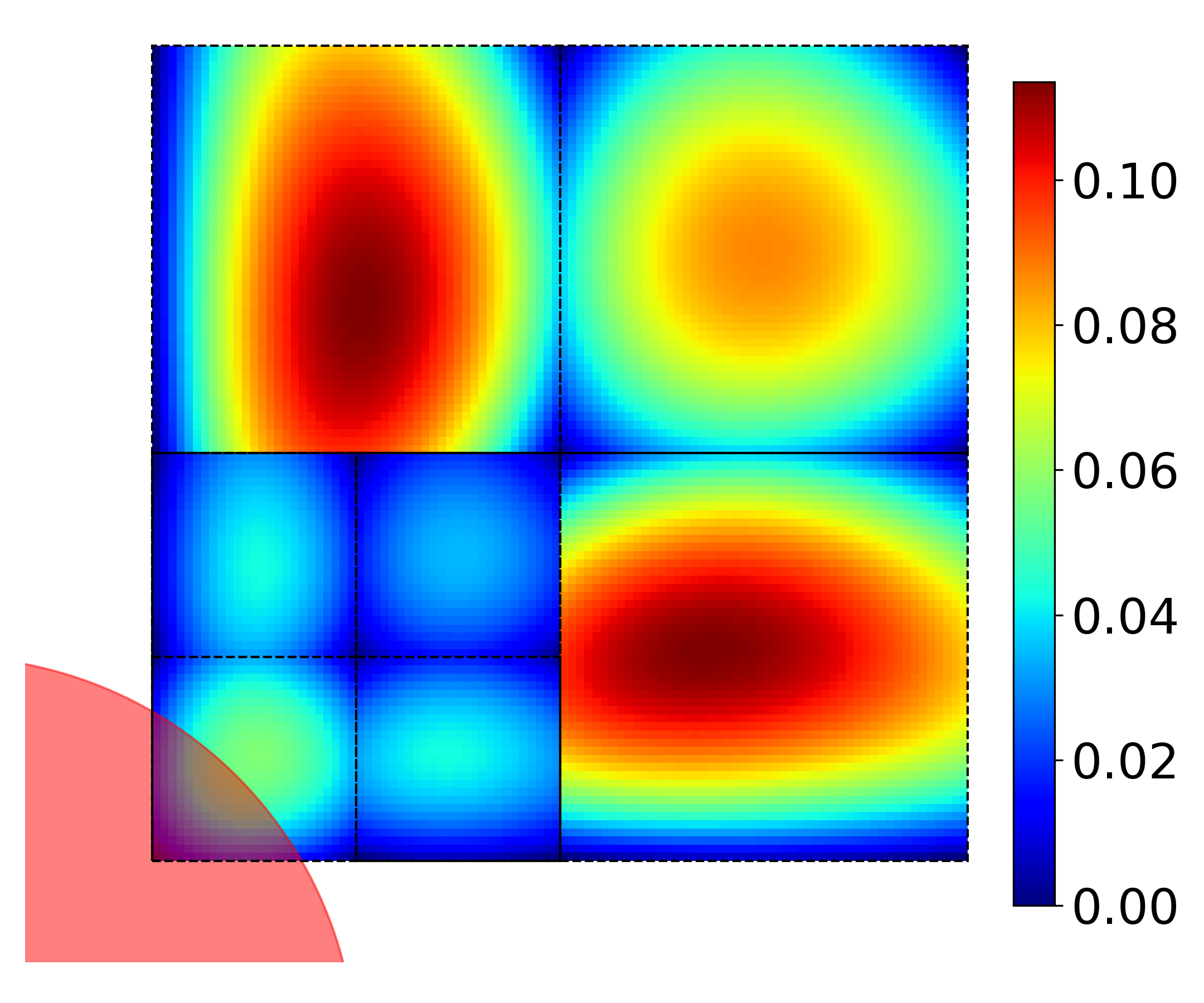}
      \caption{}
      \label{fig:comp_octree_structure_d}
    \end{subfigure}
  \caption{\small 2D visualization of SDF interpolation without gradient augmentation using (a) a sparse octree and (b) a semi-sparse octree with corresponding interpolation error shown in (c) and (d) respectively. The bottom-left red region is an obstacle containing one vertex.}
  \label{fig:comp_octree_structure}
  \vspace{-4.5ex}
\end{figure}

Fig. \ref{fig:comp_octree_structure} shows a 2D example of an SDF prior using trilinear interpolation in a sparse and a semi-sparse octree.  In a sparse octree, the SDF interpolation has more discontinuities on the octant boundaries compared to the result in a semi-sparse octree. Given a query point $\bfx$, the semi-sparse octree provides the smallest octant containing $\bfx$ that is not larger than the smallest octant found in the sparse octree. This guarantees that we can find vertices closer to $\bfx$ for computing the SDF prior because of the creation of sibling octants, which leads to a significantly smaller SDF interpolation error as indicated by Fig.~\ref{fig:comp_octree_structure_c} and Fig.~\ref{fig:comp_octree_structure_d}.

For any query position where the surface exists in the neighborhood, we can locate an octant no larger than $r\times2^{N-M}$. For queries distant to the surface, an empty large octant is sufficient for computing an accurate SDF prior using gradient-augmented interpolation, which is discussed next.

\subsubsection{Gradient-Augmented Interpolation}
\label{sec:gradient_augmented_interpolation}

\begin{figure*}
    \centering
    \begin{subfigure}[t]{0.16\linewidth}
        \includegraphics[width=\linewidth]{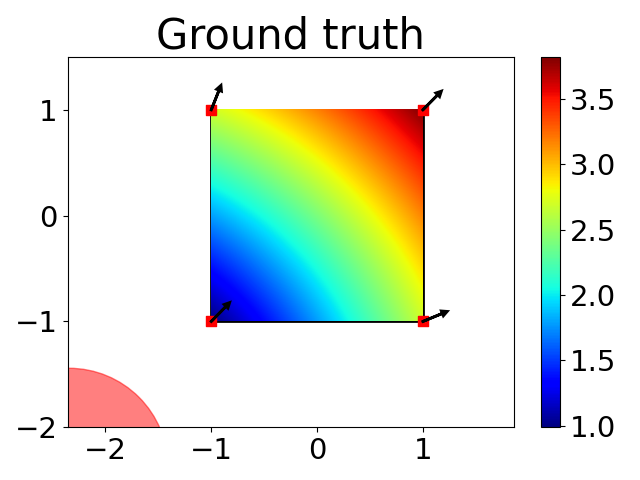}
        \vspace{-3ex}
    \end{subfigure}%
    \hfill%
    \begin{subfigure}[t]{0.16\linewidth}
        \includegraphics[width=\linewidth]{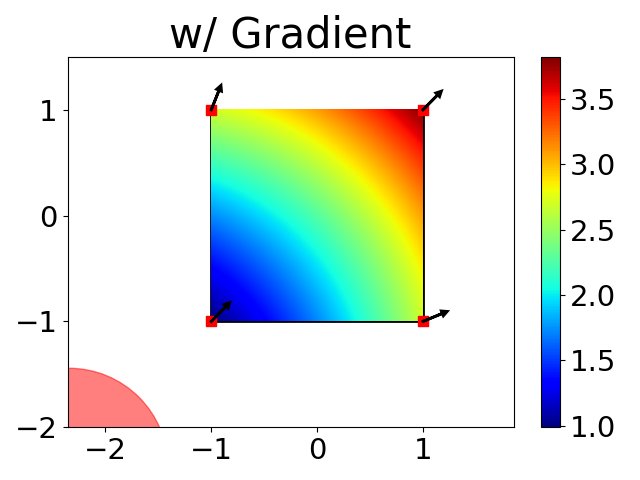}
        \vspace{-3ex}
    \end{subfigure}%
    \hfill%
    \begin{subfigure}[t]{0.16\linewidth}
        \includegraphics[width=\linewidth]{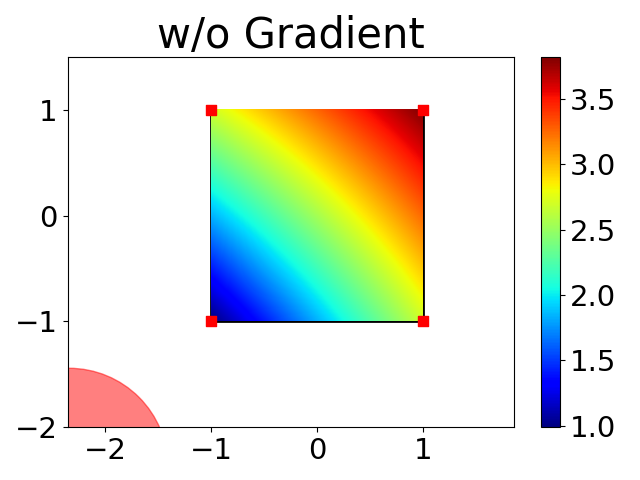}
        \vspace{-3ex}
    \end{subfigure}%
    \hfill%
    \begin{subfigure}[t]{0.16\linewidth}
        \includegraphics[width=\linewidth]{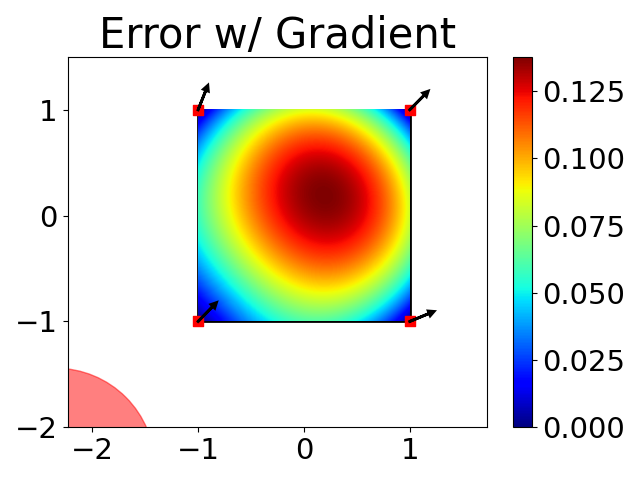}
        \vspace{-3ex}
    \end{subfigure}%
    \hfill%
    \begin{subfigure}[t]{0.16\linewidth}
        \includegraphics[width=\linewidth]{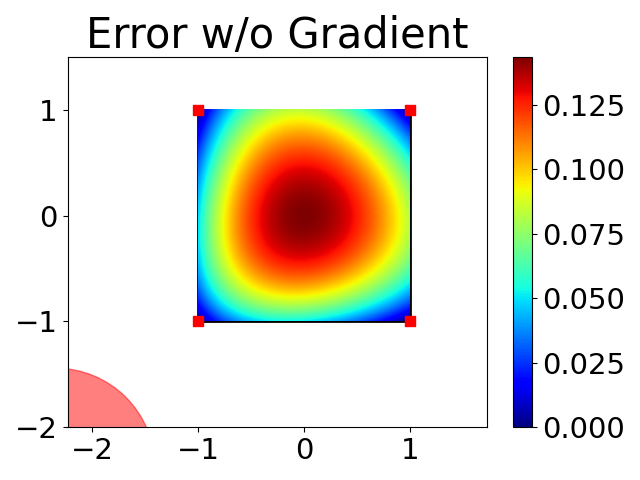}
        \vspace{-3ex}
    \end{subfigure}%
    \hfill%
    \begin{subfigure}[t]{0.16\linewidth}
        \includegraphics[width=\linewidth]{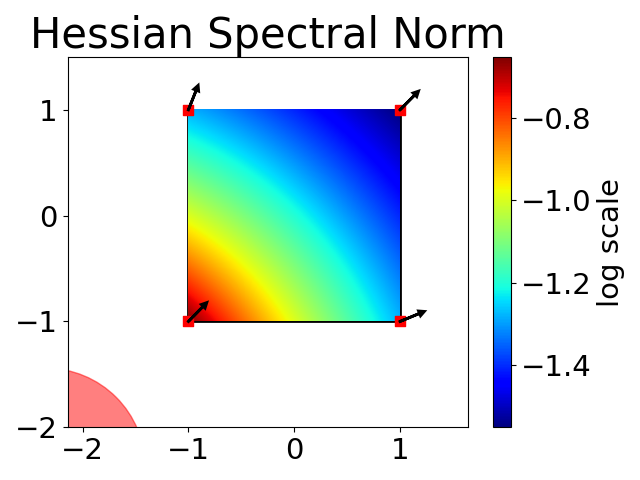}
        \vspace{-3ex}
    \end{subfigure}
    \begin{subfigure}[t]{0.16\linewidth}
        \includegraphics[width=\linewidth]{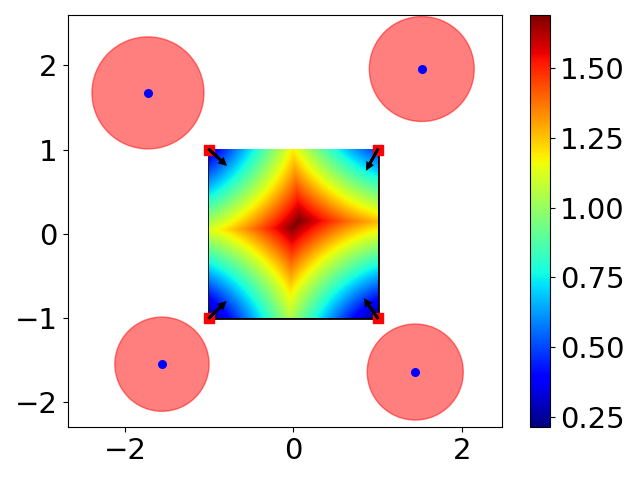}
        \vspace{-3ex}
        \caption{}
        \label{fig:comp_interpolation_a}
    \end{subfigure}%
    \hfill%
    \begin{subfigure}[t]{0.16\linewidth}
        \includegraphics[width=\linewidth]{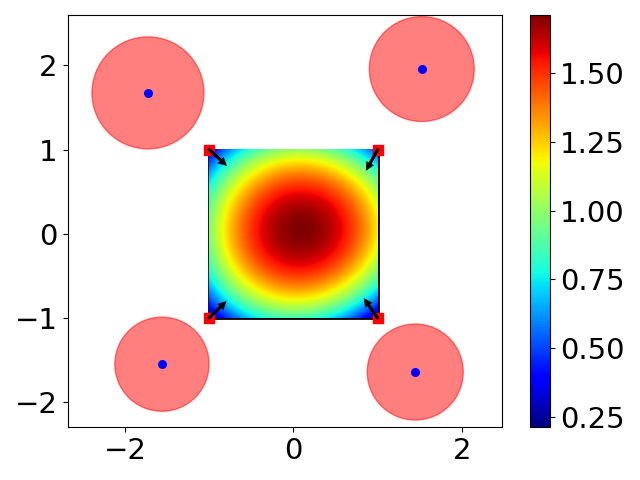}
        \vspace{-3ex}
        \caption{}
        \label{fig:comp_interpolation_b}
    \end{subfigure}%
    \hfill%
    \begin{subfigure}[t]{0.16\linewidth}
        \includegraphics[width=\linewidth]{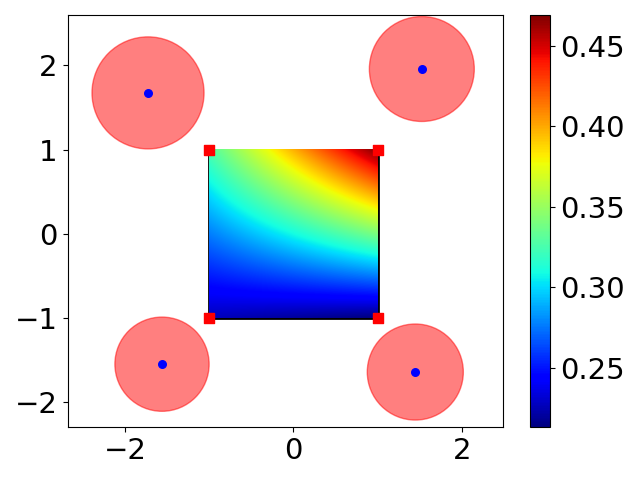}
        \vspace{-3ex}
        \caption{}
        \label{fig:comp_interpolation_c}
    \end{subfigure}%
    \hfill%
    \begin{subfigure}[t]{0.16\linewidth}
        \includegraphics[width=\linewidth]{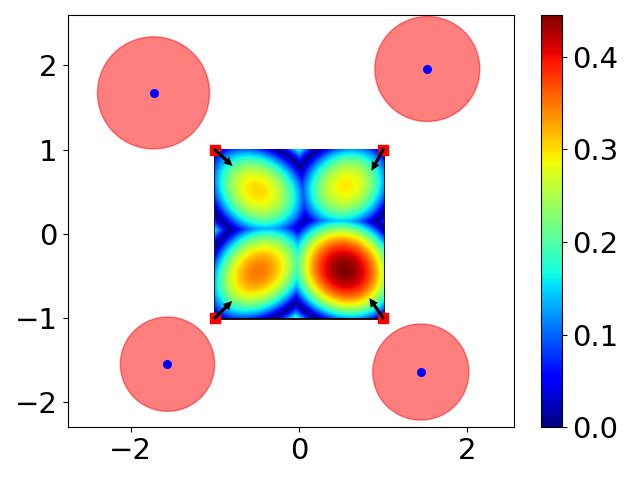}
        \vspace{-3ex}
        \caption{}
        \label{fig:comp_interpolation_d}
    \end{subfigure}%
    \hfill%
    \begin{subfigure}[t]{0.16\linewidth}
        \includegraphics[width=\linewidth]{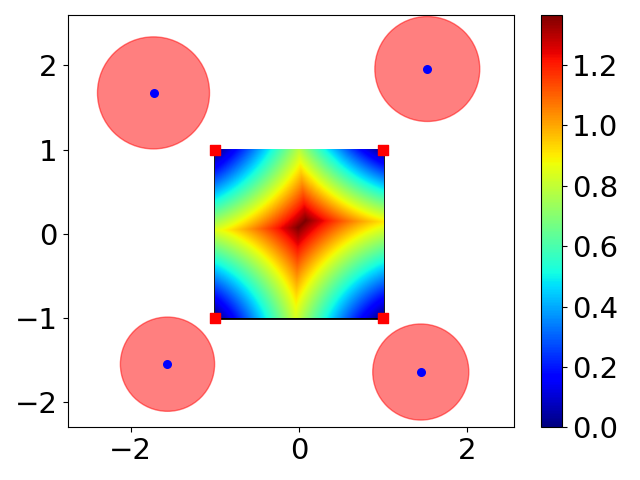}
        \vspace{-3ex}
        \caption{}
        \label{fig:comp_interpolation_e}
    \end{subfigure}%
    \hfill%
    \begin{subfigure}[t]{0.16\linewidth}
        \includegraphics[width=\linewidth]{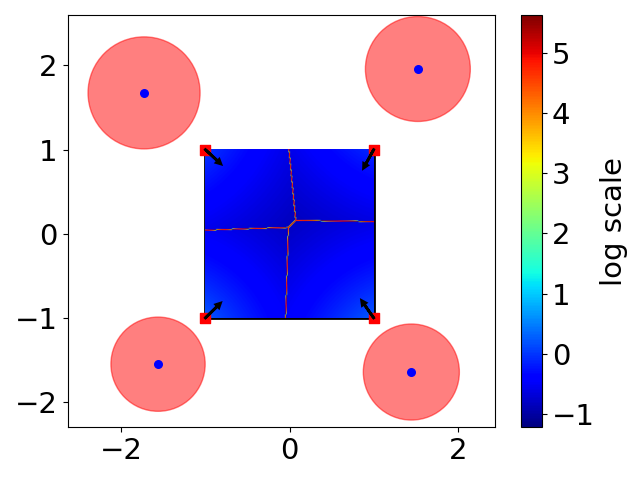}
        \vspace{-3ex}
        \caption{}
        \label{fig:comp_interpolation_f}
    \end{subfigure}
    \caption{\small 2D interpolation with and without gradient augmentation for one (red region, top row) and four obstacles (red regions, bottom row). Gradient-augmented interpolation produces a better SDF prior in (b) with smaller error in (d). Empirically, positions where the SDF gradient is not well defined (large Hessian spectral norm, (f)) have small interpolation error with gradient augmentation (d).}
    \label{fig:comp_interpolation}
    \vspace{-1em}
\end{figure*}

To achieve smaller SDF errors of the prior so that the subsequent neural network can focus on restoring the geometric details, we propose a new gradient-augmented trilinear interpolation method. Given the smallest octant that contains a query position $\bfx$, we first obtain an extrapolation result from each vertex $\bfx_k$:
\begin{equation}
    d_k(\bfx) = d_k + \bfg_k^\top (\bfx - \bfx_k),\ k\in\{1,\dots,8\}.
\end{equation}
Given the extrapolation results, we compute the gradient-augmented (ga) interpolation:
\begin{equation} \label{eq:sdf_prior}
    d_{ga}(\bfx) = \frac{1}{\gamma} \sum_{k=1}^8 w_k d_k(\bfx),\ \gamma=\sum_{k=1}^8 w_k,
\end{equation}
where $w_k = 1/|\diag{(\bfx_i - \bfx_{k})}|$ is the interpolation weight. It is a standard formulation but expressed in vector form for brevity. In contrast, the regular trilinear (tl) interpolation is:
\begin{equation} \label{eq:tl_interpolation}
    d_{tl}(\bfx) = \frac{1}{\gamma} \sum_{k=1}^8 w_k d_k,\ \gamma=\sum_{k=1}^8 w_k.
\end{equation}
Empirically, gradient-augmented interpolation generates more accurate SDF priors. Fig. \ref{fig:comp_interpolation} shows two 2D examples. Each row shows the ground truth SDF, interpolation results of using w/ and w/o gradient augmentation, the corresponding errors, and the Hessian spectral norm of SDF. As shown in Fig. \ref{fig:comp_interpolation_d} and \ref{fig:comp_interpolation_e}, gradient-augmented interpolation has smaller errors, especially when more obstacles are in the scene.
The gradient-augmented interpolation requires extra memory and computation for the gradient $\bfg_k$ and the extrapolation, respectively. However, theoretically, interpolation with gradient augmentation causes a smaller error upper bound.

\vspace{-1ex}
\begin{proposition}
    Consider an octant $\calV \subset \bbR^3$ of size $L$. Assume that the SDF $d(\bfx)$ is twice differentiable and the spectral norm of its Hessian is bounded:
    \begin{equation}\label{eq:bounded_hessian}
        M := \sup_{\bfx \in \calV} \| \nabla^2 d(\bfx) \|_2 < \infty.
    \end{equation}
    Assume that each vertex $\bfx_k$ has ground-truth SDF value $d_k=d(\bfx_k)$ and gradient $\bfg_k=\nabla d(\bfx_k)$. Then, given arbitrary $\bfx \in \calV$, the errors of gradient-augmented interpolation in \eqref{eq:sdf_prior} and trilinear interpolation in \eqref{eq:tl_interpolation} satisfy:
    \begin{equation}
    \begin{aligned}
    e_{ga}(\bfx) &= |d_{ga}(\bfx)-d(\bfx)| \leq \bar{e}_{ga} = \frac{3ML^2}{8},\\
    e_{tl}(\bfx) &= |d_{tl}(\bfx)-d(\bfx)| \leq \bar{e}_{tl} = \frac{\sqrt{3}L}{2}.
    \end{aligned}
    \end{equation}
\end{proposition}
\vspace{-1ex}
\begin{proof}
    For each octant vertex $\bfx_k$, by Taylor's theorem:
    $d(\bfx) = d_k + \bfg_k^\top (\bfx - \bfx_k) + \frac{1}{2}(\bfx-\bfx_k)^\top \nabla^2 d(\bfxi_k) (\bfx - \bfx_k)$,
    for some $\bfxi_k$ on the line segment joining $\bfx$ and $\bfx_k$. Using \eqref{eq:sdf_prior} and \eqref{eq:bounded_hessian}, the gradient-augmented interpolation error satisfies:
    \begin{equation*}
        e_{ga}(\bfx) = \frac{1}{2\gamma} \biggl| \sum_{k=1}^8 w_k (\bfx-\bfx_k)^\top \nabla^2 d(\bfxi_k) (\bfx - \bfx_k) \biggr|
    \end{equation*}
    \begin{equation}
        \leq \frac{M}{2\gamma} \sum_{k=1}^8 w_k \|\bfx-\bfx_k\|_2^2 \le \frac{3ML^2}{8} = \bar{e}_{ga},
    \end{equation}
    where equality holds when $\bfx$ is the octant center so that $\frac{1}{\gamma}\sum_{k=1}^8 w_k \|\bfx-\bfx_k\|_2^2=3L^2/4$.
    Similarly, for the error of the regular trilinear interpolation, we have:
    \begin{equation}
        d(\bfx) = d_k + \nabla d(\bfzeta_k)^\top (\bfx-\bfx_k)
    \end{equation}
    for some $\bfzeta_k$ on the line segment joining $\bfx$ and $\bfx_k$. Then, since $\|\nabla d(\bfzeta_k)\|=1$ and using \eqref{eq:tl_interpolation}, the error satisfies:
    \begin{align}
        e_{tl}(\bfx) &= \biggl|\frac{1}{\gamma}\sum_{k=1}^K w_k \nabla d(\bfxi_k)^\top (\bfx-\bfx_k)\biggr| \\
        &\leq \frac{1}{\gamma}\sum_{k=1}^K w_k \|\bfx - \bfx_k\|_2 \le \frac{\sqrt{3}L}{2} = \bar{e}_{tl}. \notag \qedhere
    \end{align}
    \vspace{-1em}
\end{proof}

When an octant does not contain positions where the gradient is not well defined (e.g., the medial axes where the closest point projection is not unique), \eqref{eq:bounded_hessian} holds. For positions without well-defined gradients, although the Hessian norm blows up mathematically, gradient-augmented interpolation has $e_{ga}(\bfx)$ significantly lower than $\bar{e}_{ga}$ based on empirical observation. The second row of Fig. \ref{fig:comp_interpolation_d} and \ref{fig:comp_interpolation_f} shows an example of such cases, that the Hessian spectral norm is large on the medial axes, but gradient-augmented interpolation has smaller errors. Since we are looking for an upper bound on the error, we ignore such cases.

As shown in Fig. \ref{fig:comp_interpolation}, the gradient-augmented interpolation has smaller errors than without gradient augmentation.
Empirically, as shown in Fig. \ref{fig:comp_interpolation_f}, $M\ll 1$ so that $\bar{e}_{ga} / \bar{e}_{tl} = {\sqrt{3}ML}/{4} < 1$.
Especially, when the octant is surrounded by multiple obstacles, the SDF prior values stored at the vertices are smaller than the ground truth SDF values inside the octant. This makes SDF priors obtained from interpolation without gradient augmentation no larger than the vertex SDF values, leading to significantly larger errors.

Hence, the prior network $T(\bfx;\theta)$ of our method is a semi-sparse octree where each vertex has an estimate of SDF and gradient, i.e., $\theta=\{d_k, \bfg_k\}_{k=1}^K$, which are learnable parameters optimized together with the residual network. In the experiments, we maintain a semi-sparse octree for each scene with $N=8$, $M=5$ and $r=10$ cm.

\subsection{SDF Residual via Neural Feature Decoding}
\label{sec:sdf_residual}

The SDF prior's accuracy is limited by the octree resolution, lacking geometric details.
To achieve high fidelity, \methodname learns a residual correction to the SDF prior via a neural network $R(d_\text{ga}(\bfx), \bfx; \{\bff_i\}_i, \beta)$, a composition of octree feature interpolation $\bff(x) = \sum_k w_k \bff_k$ with implicit neural features $\bff_i \in \bbR^F$ stored at octree vertices and an MLP decoder $D(d, \bff(x); \beta)$.

To enable continual learning as the sensor moves, we assign each octant vertex an implicit feature $\bff_i$. Octree expansion automatically adds more features $\bff_i$ to near-surface regions in smaller octants.
The features $\bff_i$ are initialized to zero and optimized together with the network weights $\beta$.

As shown in Fig.~\hyperref[fig:method_overview]{2d}, for each query point $\bfx$, we locate the leaf octant containing $\bfx$ and compute the interpolated neural feature $\bff(\bfx) = \sum_{k=1}^8 w_k \bff_k$ with $\bff_k$ at the octant vertices, where $w_k$ is the same weight used in \eqref{eq:sdf_prior}.
The feature $\bff(\bfx)$ and the prior $d_\text{ga}(\bfx)$ are fed into a decoder to predict the SDF residual $\delta_d(\bfx)=D(d_\text{ga}(\bfx), f(\bfx);\beta)$.
In the evaluation, we have $F=3$ and the MLP has two 32-dim hidden layers with LeakyReLU activation.

\vspace{-0.5ex}
\subsection{Key Frame Selection}
\label{sec:key_frame_selection}

Real-time SDF learning requires a compact, representative set of training frames. We adopt the key-frame selection of $H_2$-Mapping \cite{jiang_h2-mapping_2023}. As shown in Fig.~\hyperref[fig:method_overview]{2a}, we insert a new sensor frame when the newly observed area compared with the last key frame is large enough, i.e., $\frac{|V_c \cap V_l|}{|V_c \cup V_l|} > c_{\min}$, where $V_c$ and $V_l$ are the sets of surface octants observed by the current frame and the last inserted key frame, respectively, and $c_{\min} \in [0, 1]$ is a threshold. This ensures the frames cover the observed surface with little overlap.

As key frames accumulate, we select only $W$ of them to maintain real-time operation. We incrementally select the frame that observes the most octants, mask them out, and repeat until $W$ frames are collected. If all octants are masked, we reset the mask except those from the last selected key frame, and continue. This ensures the selected key frames maximally cover the scene. In the evaluation, we set $W=8$.

\vspace{-1ex}
\subsection{Dataset Generation}
\label{sec:dataset_generation}

During online training, it is important to generate a high-quality dataset consisting of a small number of representative samples.
At time step $t$, suppose the set of selected key frame time steps is $\calT=\{k, 1\le k \le t\}, |\calT|\le W$.
For each frame $\calP_{i \in \calT \cup \{t\}}$, we randomly choose $\lfloor N / |\calT \cup \{t\}| \rfloor$ rays $\{\bfo_j=\bfo_i, \bfq_j \in \calP_i\}_{j=1}^N$ to generate samples. As shown in Fig.~\hyperref[fig:method_overview]{2b}, we generate three types of points for training:

\textit{1) Free-space Points}:
To learn the SDF in free space, we sample free-space points $\calP_F$ by drawing $n_F$ points $\{\bfx_n\}_{n=1}^{n_F}$ along each ray $j$: $\bfx_n = \bfo_j + \lambda (\bfq_j - \bfo_j)$, where $\lambda$ is drawn from the uniform distribution $\calU(\delta, 1-\delta)$ with margin $\delta > 0$;

\textit{2) Surface Samples and Perturbed Points}:
To provide supervision for the surface reconstruction, we generate $\calP_S$ by collecting the surface point $\bfq_j$ of each ray $j$ and the perturbed points $\calP_P$ by sampling $n_P$ points $\{\bfx_n\}_{n=1}^{n_P}$ along each ray $j$, i.e., $\bfx_n=\bfq_j + \alpha (\bfq_j - \bfo_j) / \|\bfq_j - \bfo_j\|_2$, where $\alpha$ has uniform distribution on $(-3\sigma, \sigma) \cup (\sigma, 3\sigma)$.

\textit{3) Ground Truth SDF Computation}:
For surface points $\calP_S$, we have ground truth SDF $d(\bfx)=0, \bfx \in \calP_S$. For perturbed points and free-space points, we approximate the ground truth as $\tilde{d}(\bfx)=\operatorname{sign}(\bfx)\min_{\bfy\in\calP_S}\|\bfx-\bfy\|_2$.

In our experiments, we have $W=8$, $N=20480$, $\delta=0.05$, $\sigma=0.06$, $n_F=1$, and $n_P=2$.

\subsection{Loss Functions}
\label{sec:loss_functions}

As shown in Fig.~\hyperref[fig:method_overview]{2e}, the SDF prior $d_{ga}(\bfx)=T(\bfx;\theta)$ and the neural residual $\delta_d(\bfx)=R(d_\text{ga}(\bfx), \bfx; \{\bff_i\}_i, \beta)$ are combined together to obtain a final SDF prediction:
\begin{equation}
    \hat{d}(\bfx) = d_{ga}(\bfx) + \delta_d(\bfx).
\end{equation}
It is important to design appropriate loss functions to train the octree and neural network parameters $\theta$, $\beta$.

\subsubsection{Reconstruction Loss}
We apply an L1 loss over the surface points $\calP_S$ and the perturbation points $\calP_P$ to capture the surface geometry, which is critical for accurate 3D reconstruction:
\begin{equation} \label{eq:loss_recon}
\begin{aligned}
    &\calL_{\text{recon}}=\underbrace{\frac{w_\text{recon}^S}{|\calP_S|}\sum_{\bfx\in\calP_S}\left|\hat{d}(\bfx)\right|}_{\calL_{\text{surface}}} + \underbrace{\frac{w_\text{recon}^P}{|\calP_P|}\sum_{\bfx\in\calP_P} \left(c_j^{\text{lo}} + c_j^{\text{up}}\right)}_{\calL_{\text{perturbation}}}, \\
    &c_j^{\text{lo}} = \max\left(e^{\eta (\underline{b}_j - |\hat{d}(\bfx)|)}-1, 0\right),
    c_j^{\text{up}} = \max\left(|\hat{d}(\bfx)| - \overline{b}_j, 0\right),
\end{aligned}
\end{equation}
where $w_\text{recon}^S$ and $w_\text{recon}^P$ are the corresponding weights and $\underline{b}_j$ and $\overline{b}_j$ are lower and upper bounds, respectively. We set $\eta=10$. For perturbed points, $\underline{b}_j=\sigma$, $\overline{b}_j=|\tilde{d}(\bfx)|$.

\subsubsection{Eikonal Loss}
To enforce the Eikonal property in \eqref{eq:sdf_constraints}, we apply another L1 loss for the gradient norm:
\begin{equation} \label{eq:loss_eik}
\calL_\text{eik} = \sum_{k\in\{S, P, F\}} \frac{w_\text{eik}^k}{|\calP_k|}\sum_{\bfx\in\calP_k} \left| \|\hat{\bfg}(\bfx)\|_2 - 1 \right|,
\end{equation}
over three different kinds of samples $\calP_S$, $\calP_P$ and $\calP_F$,
where $w_\text{eik}^k$ is the weight for the corresponding sample set $\calP_k$.
Here, $\hat{\bfg}(\bfx)$ is obtained from numerical differentiation instead of the auto gradient graph because the numerical differentiation involves more positions $\bfx \pm \epsilon \bfe_i$ for each dimension $i$, which helps the network converge and shows better training stability when the SDF gradient is not well defined at certain positions.
Although the gradient prior $\bfg_k$ may be constrained to a unit vector directly, we found that applying the Eikonal loss to $\hat{\bfg}(\bfx)$, which depends on $\bfg_k$, performs better.

\subsubsection{Projection Loss}
Although the Eikonal loss $\calL_\text{eik}$ enforces the gradient magnitude, the supervision for the gradient direction and SDF in the distant space is still missing. Hence, we propose a projection loss for free-space points $\calP_F$ that are collected along rays:
\begin{equation}
    \calL_\text{proj} = \frac{w_\text{proj}}{|\calP_F|} \sum_{\bfx \in \calP_F} \left| \hat{d}(\bfx) - \tilde{d}(\bfx) \right|.
\end{equation}
Although the above loss has a form similar to \eqref{eq:loss_recon}, we call it \textit{projection} loss because $\tilde{d}(\bfx)$ is actually a loose upper bound for the ground truth SDF $d(\bfx)$. The purpose of this loss is not to make the model predict $\tilde{d}(\bfx)$ exactly at $\bfx$ but to provide the implicit supervision of the gradient direction so that it speeds up the convergence of other loss functions.

In our experiments, we set $w_\text{recon}^S = 1000$, $w_\text{recon}^P = 200$, $w_\text{eik}^S = w_\text{eik}^F = 10$, $w_\text{eik}^P = 3$, and $w_\text{proj} = 100$.
\section{Evaluation}
\label{sec:evaluation}

In this section, we compare \methodname with four baselines: Voxblox \cite{oleynikova_voxblox_2017}, \htwomapping \cite{jiang_h2-mapping_2023}, PIN-SLAM \cite{pan_pin-slam_2024}, and HIO-SDF \cite{hio-sdf_2024}. We first examine mesh and SDF reconstructions as a qualitative comparison, then quantitatively evaluate the methods using different metrics. In addition, we perform an ablation study to evaluate the contribution of each component in our method. We use the Replica dataset \cite{replica19arxiv}, which provides eight synthesized scenes with one trajectory per scene to compare \methodname with the baselines. We also test all methods on real-data from the Newer College dataset \cite{newercollege2021}. For all methods and datasets, we use the ground truth poses. For PIN-SLAM, we disable its localization module.

\subsection{Qualitative Results}
\subsubsection{Mesh Reconstruction}

\begin{figure*}[t]
    \begin{center}
        \begin{subfigure}[t]{0.16\linewidth}
            \includegraphics[width=\linewidth]{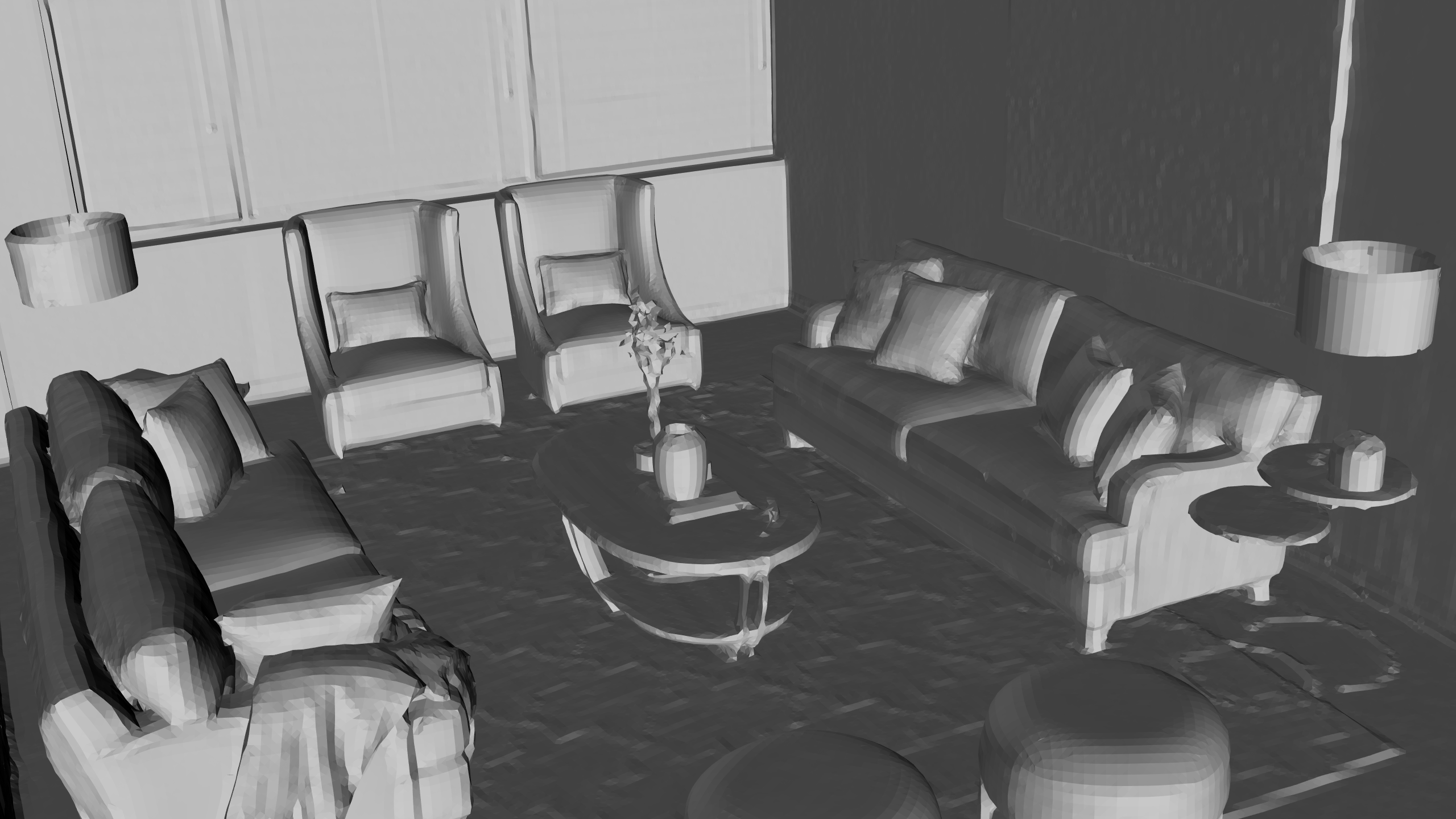}
        \end{subfigure}
        \begin{subfigure}[t]{0.16\linewidth}
            \includegraphics[width=\linewidth]{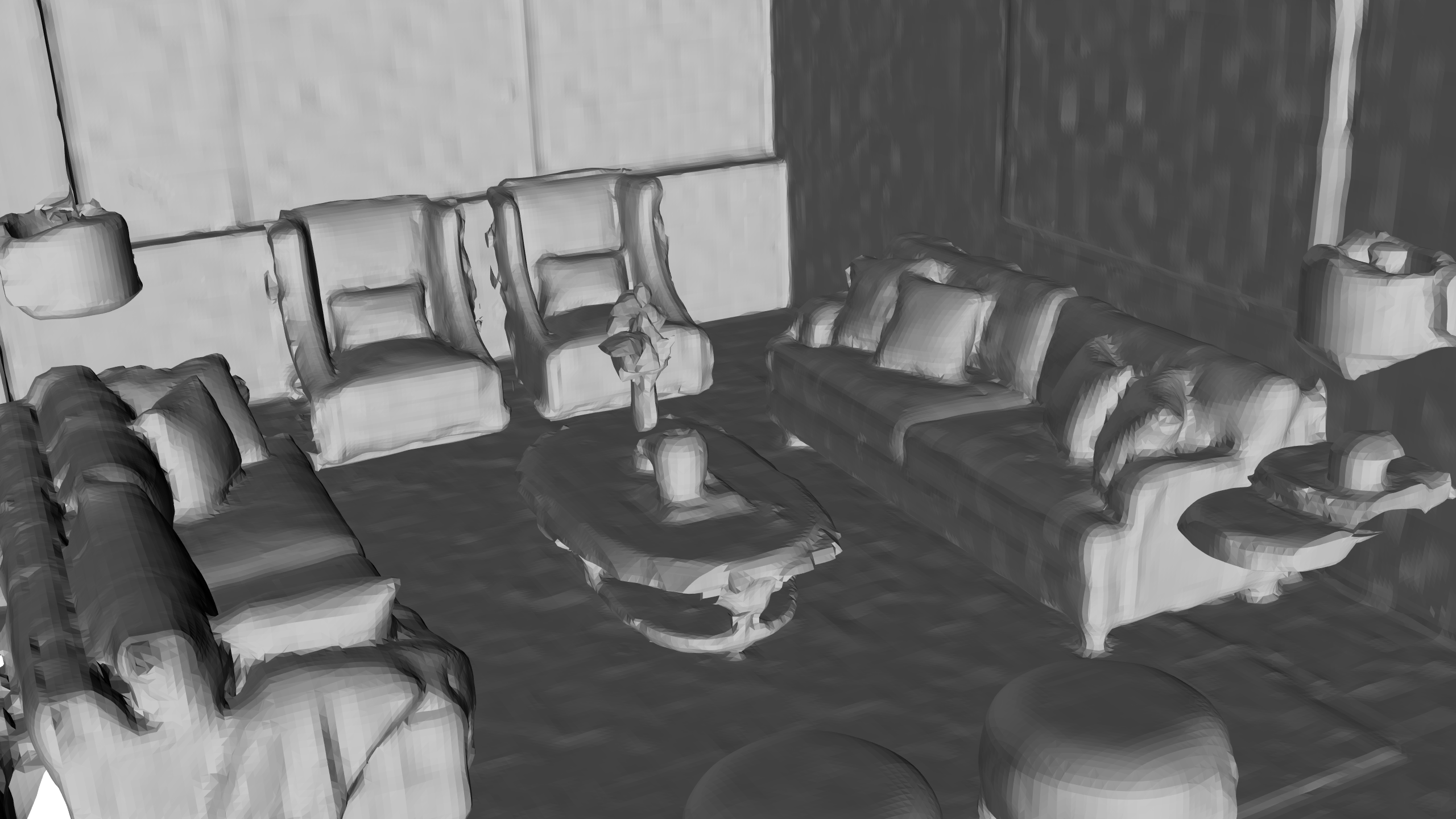}
        \end{subfigure}
        \begin{subfigure}[t]{0.16\linewidth}
            \includegraphics[width=\linewidth]{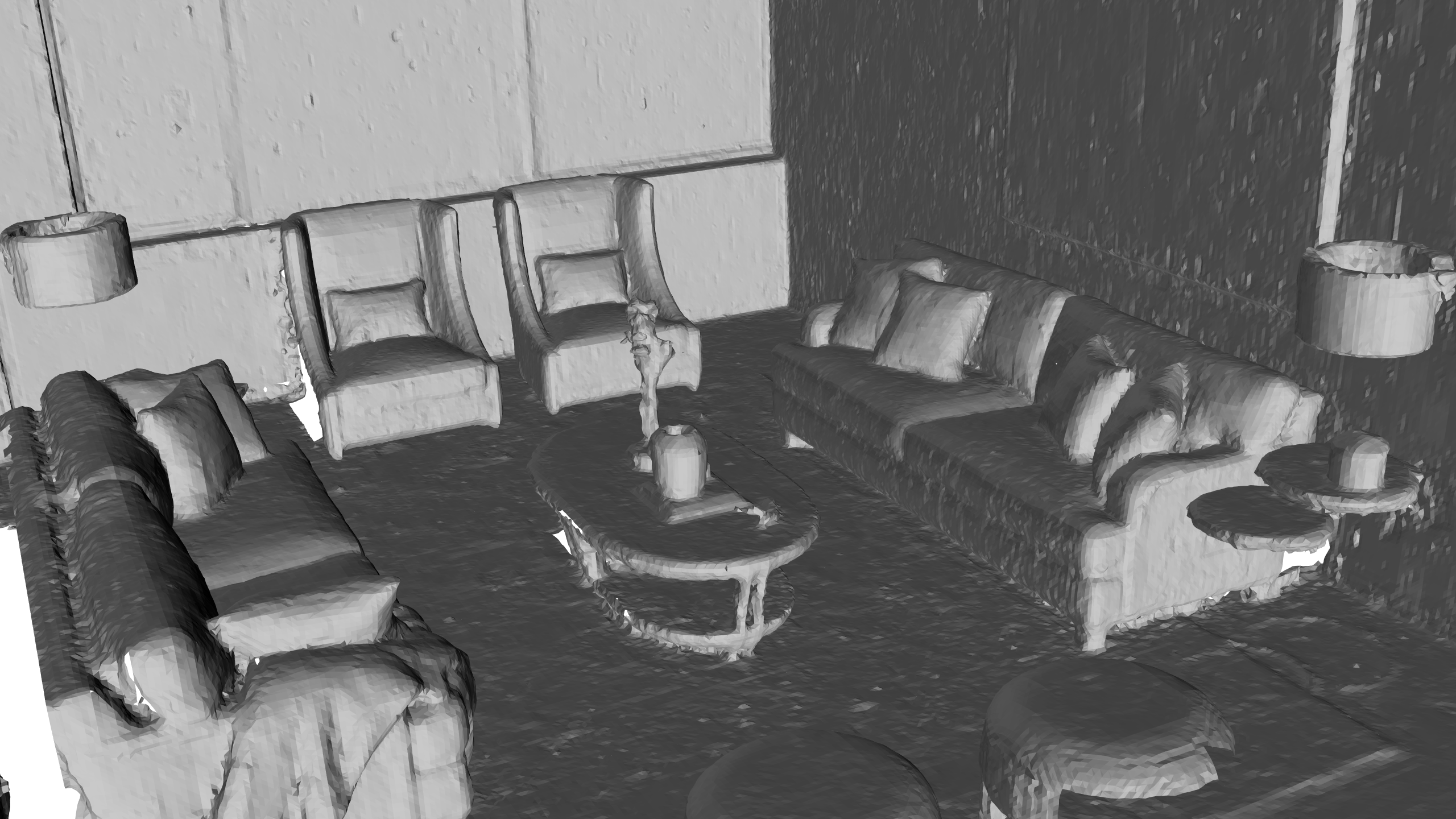}
        \end{subfigure}
        \begin{subfigure}[t]{0.16\linewidth}
            \includegraphics[width=\linewidth]{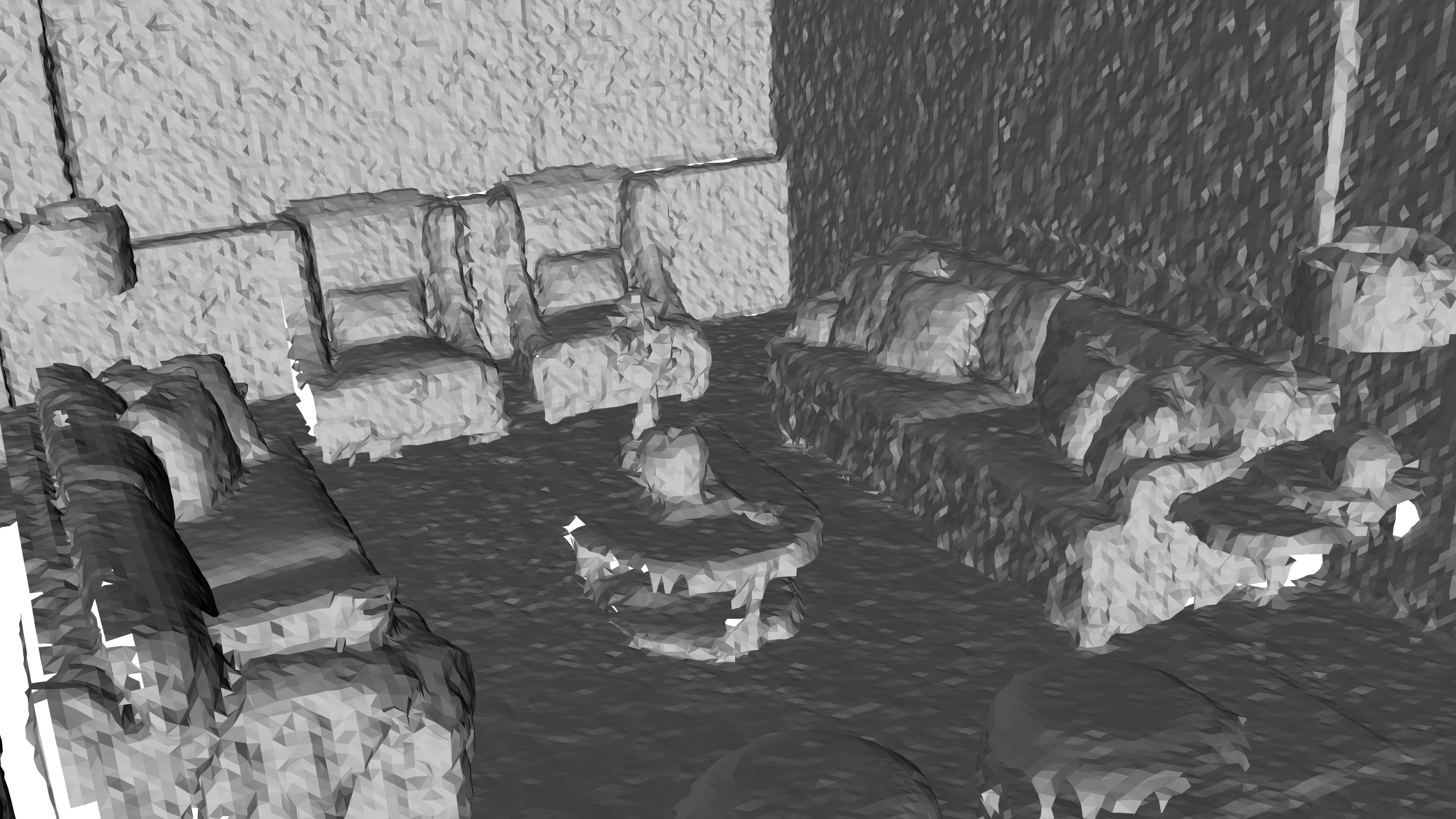}
        \end{subfigure}
        \begin{subfigure}[t]{0.16\linewidth}
            \includegraphics[width=\linewidth]{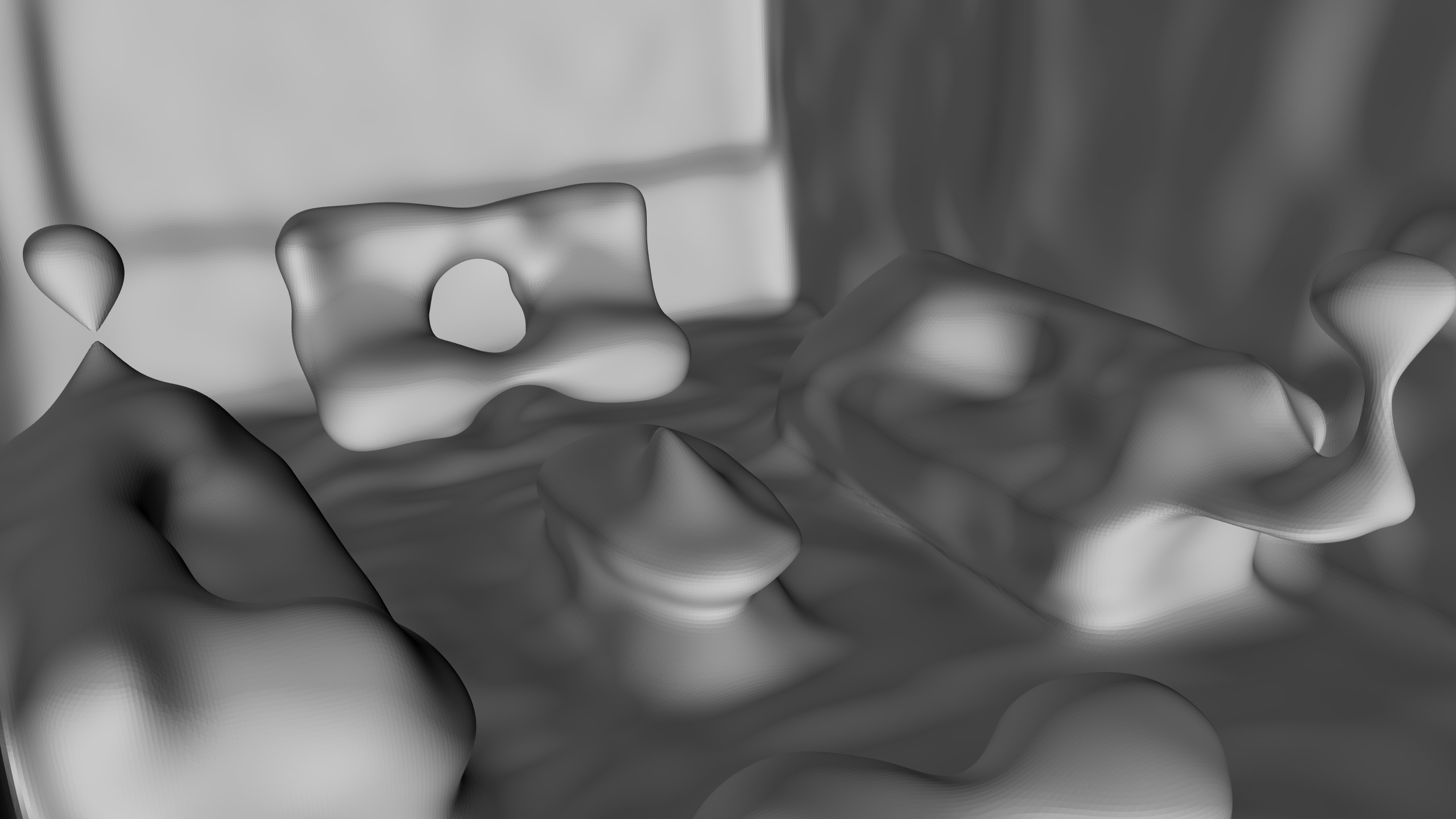}
        \end{subfigure}
        \begin{subfigure}[t]{0.16\linewidth}
            \includegraphics[width=\linewidth]{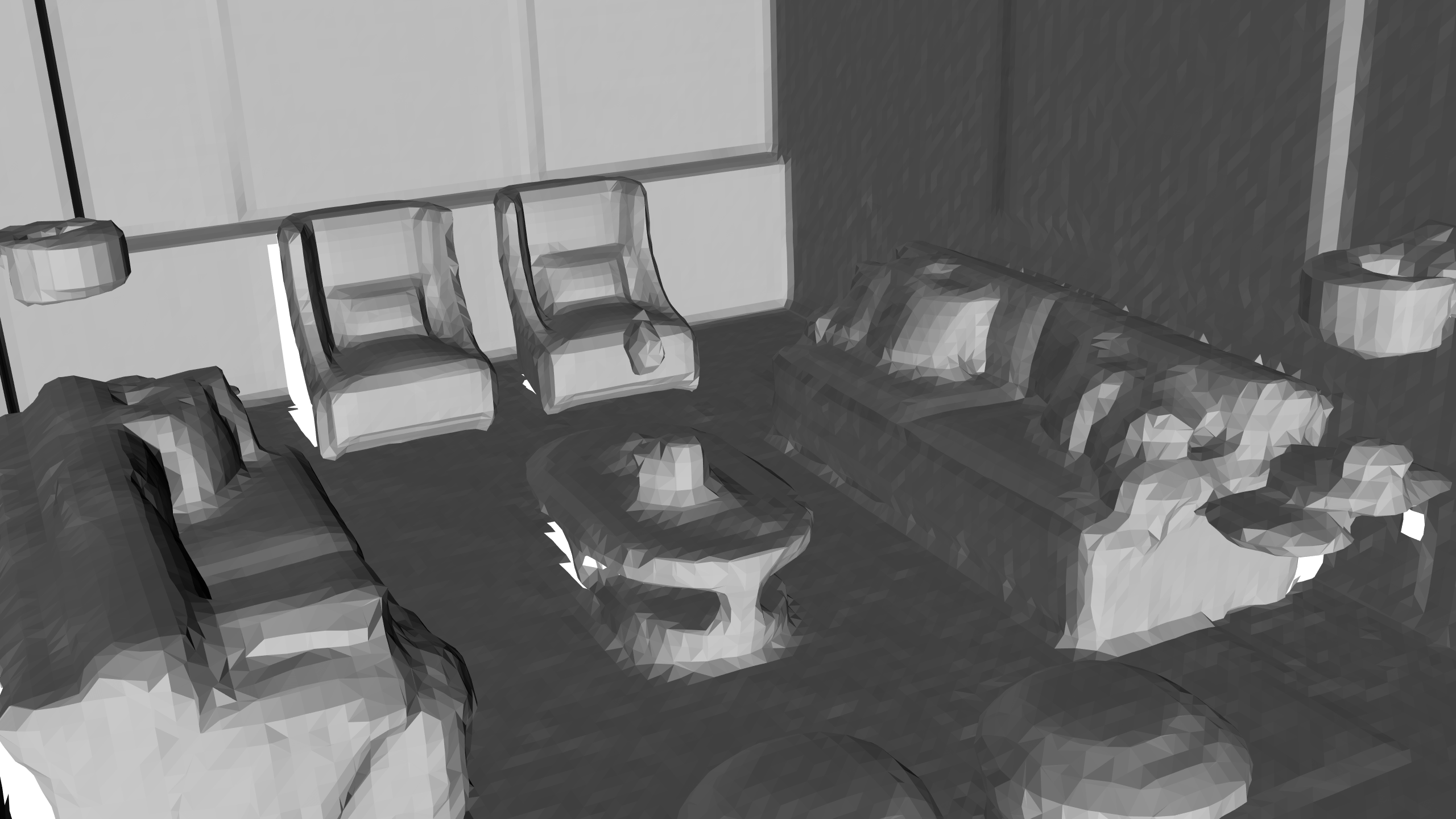}
        \end{subfigure}
        \begin{subfigure}[t]{0.16\linewidth}
            \includegraphics[width=\linewidth]{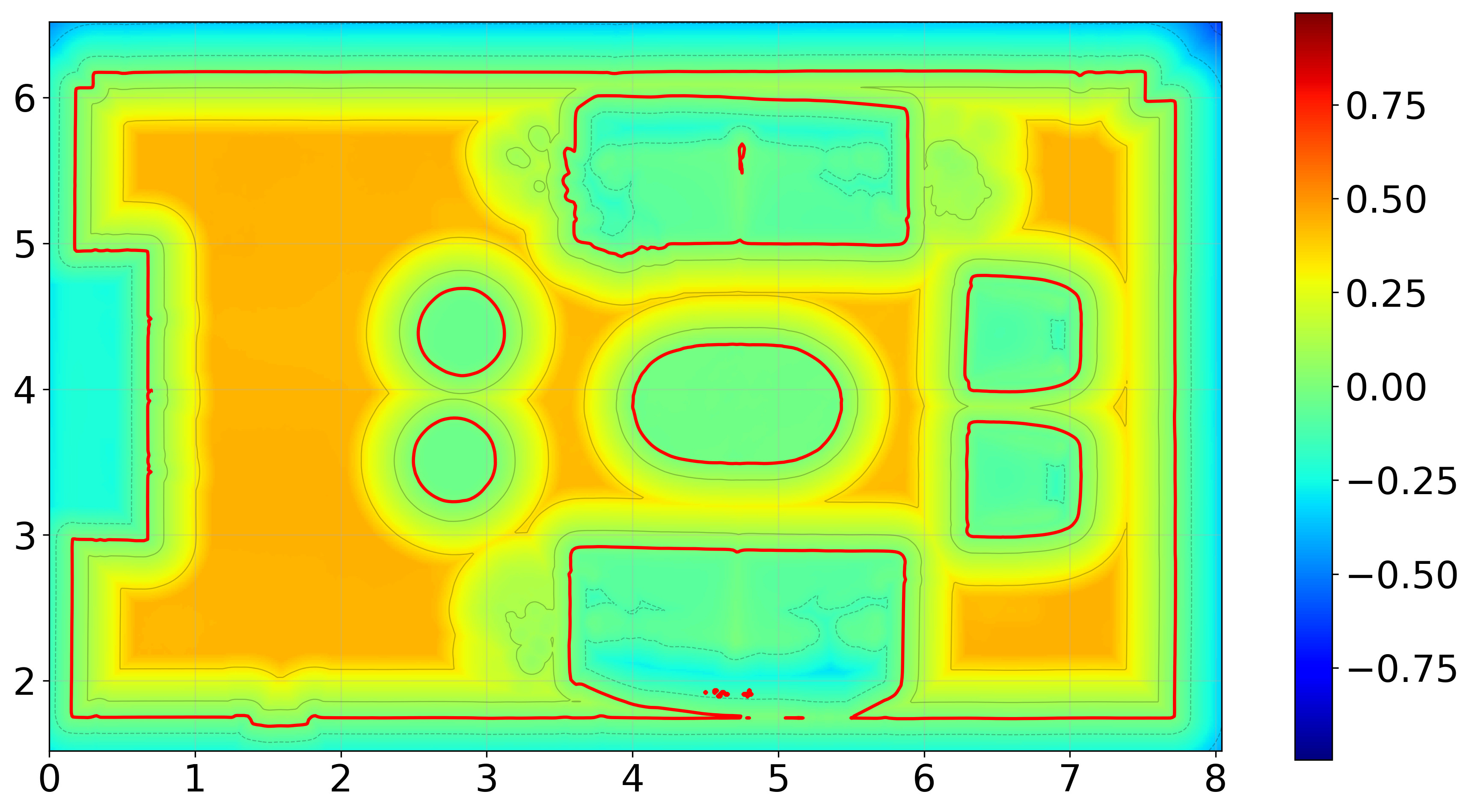}
            \caption{\footnotesize Ground Truth}
        \end{subfigure}
        \begin{subfigure}[t]{0.16\linewidth}
            \includegraphics[width=\linewidth]{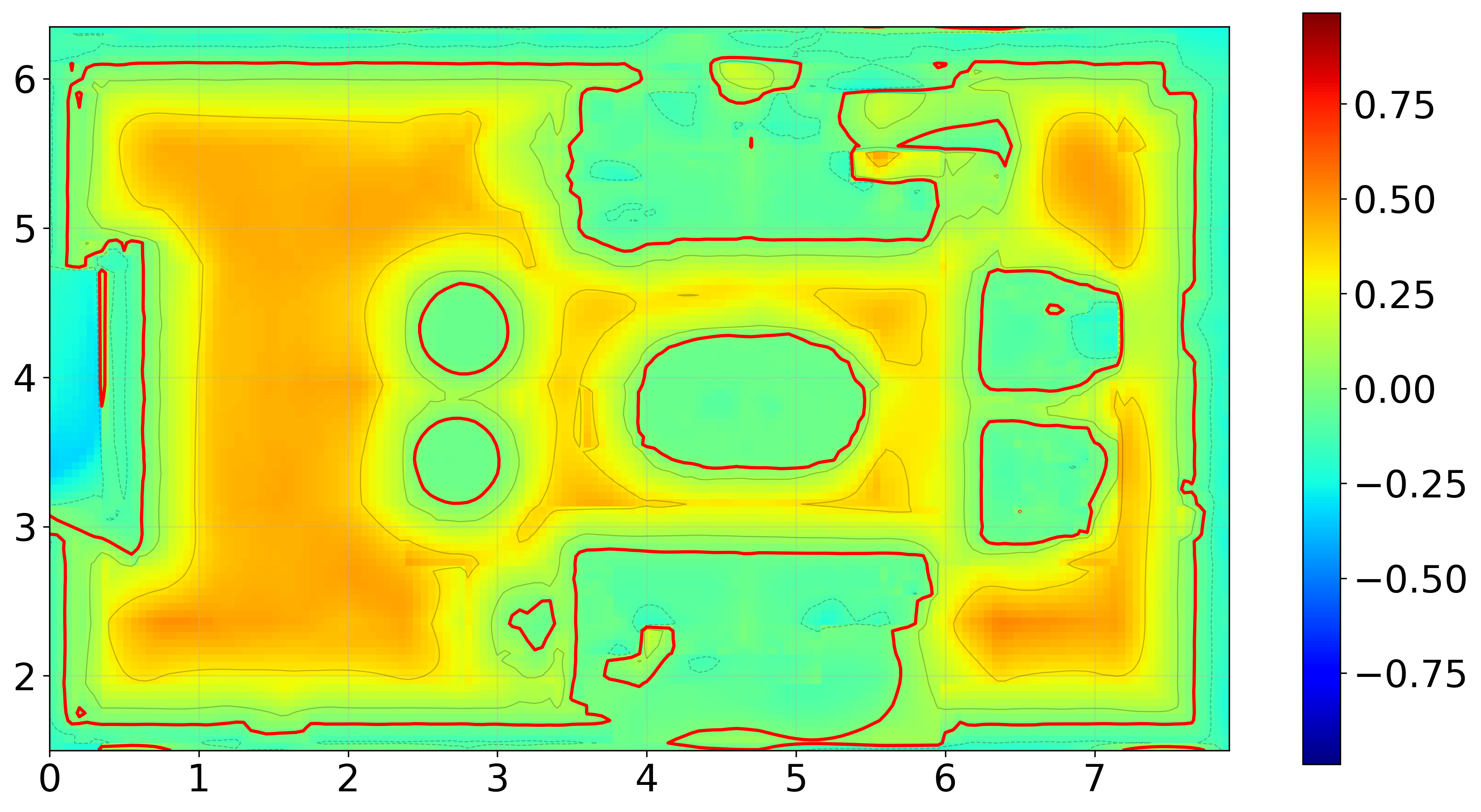}
            \caption{\footnotesize Ours}
        \end{subfigure}
        \begin{subfigure}[t]{0.16\linewidth}
            \includegraphics[width=\linewidth]{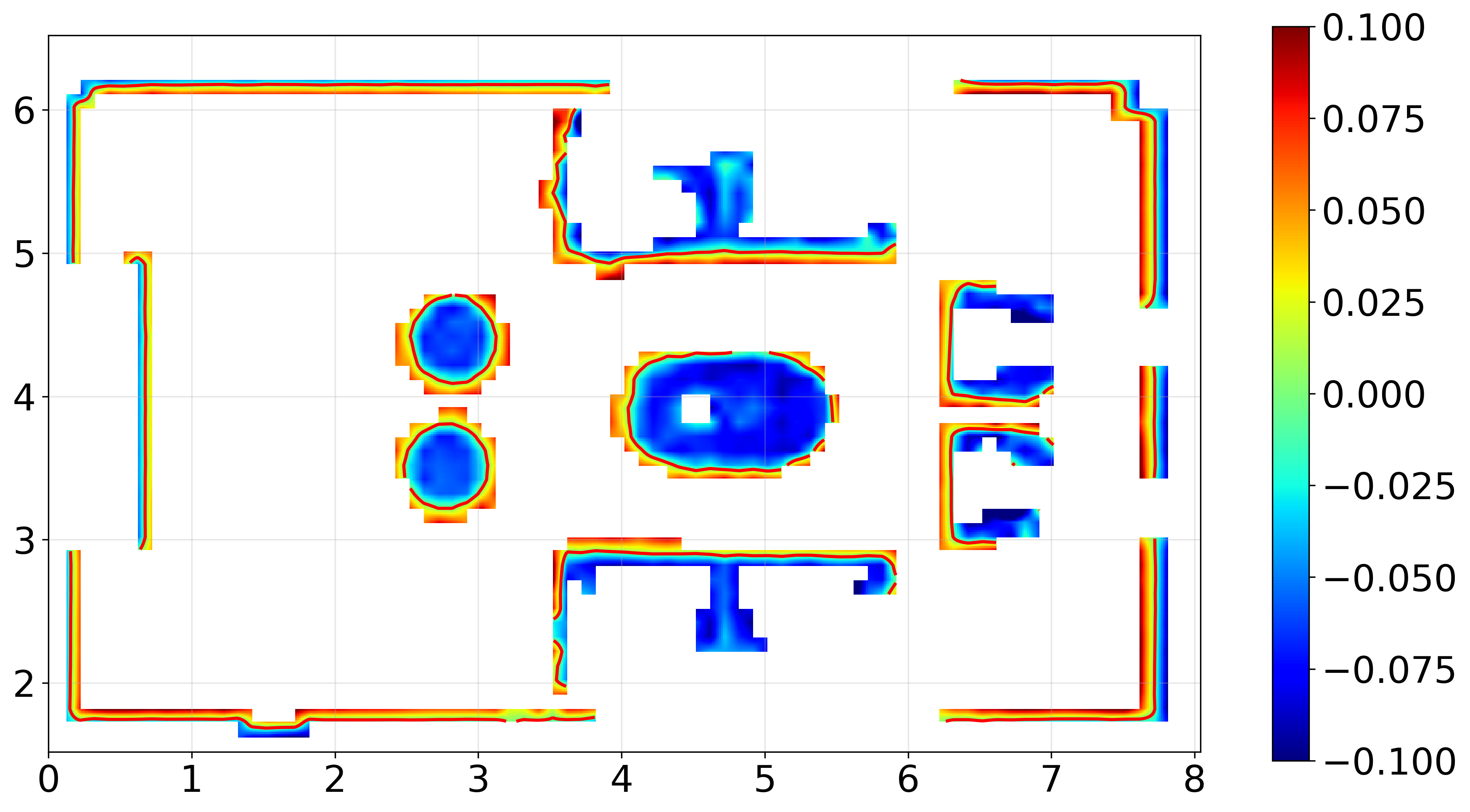}
            \caption{\footnotesize \htwomapping \cite{jiang_h2-mapping_2023}}
        \end{subfigure}
        \begin{subfigure}[t]{0.16\linewidth}
            \includegraphics[width=\linewidth]{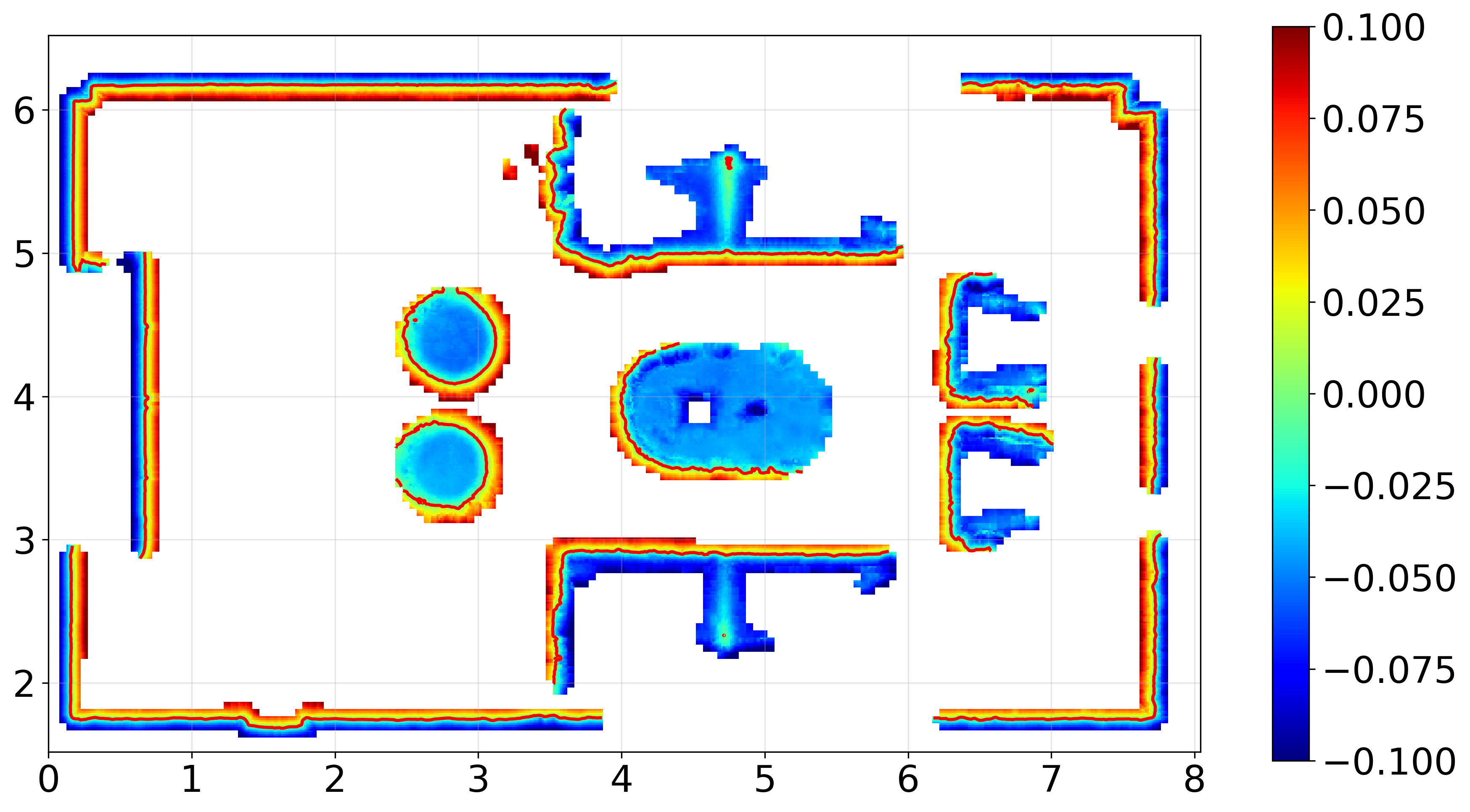}
            \caption{\footnotesize PIN-SLAM \cite{pan_pin-slam_2024}}
        \end{subfigure}
        \begin{subfigure}[t]{0.16\linewidth}
            \includegraphics[width=\linewidth]{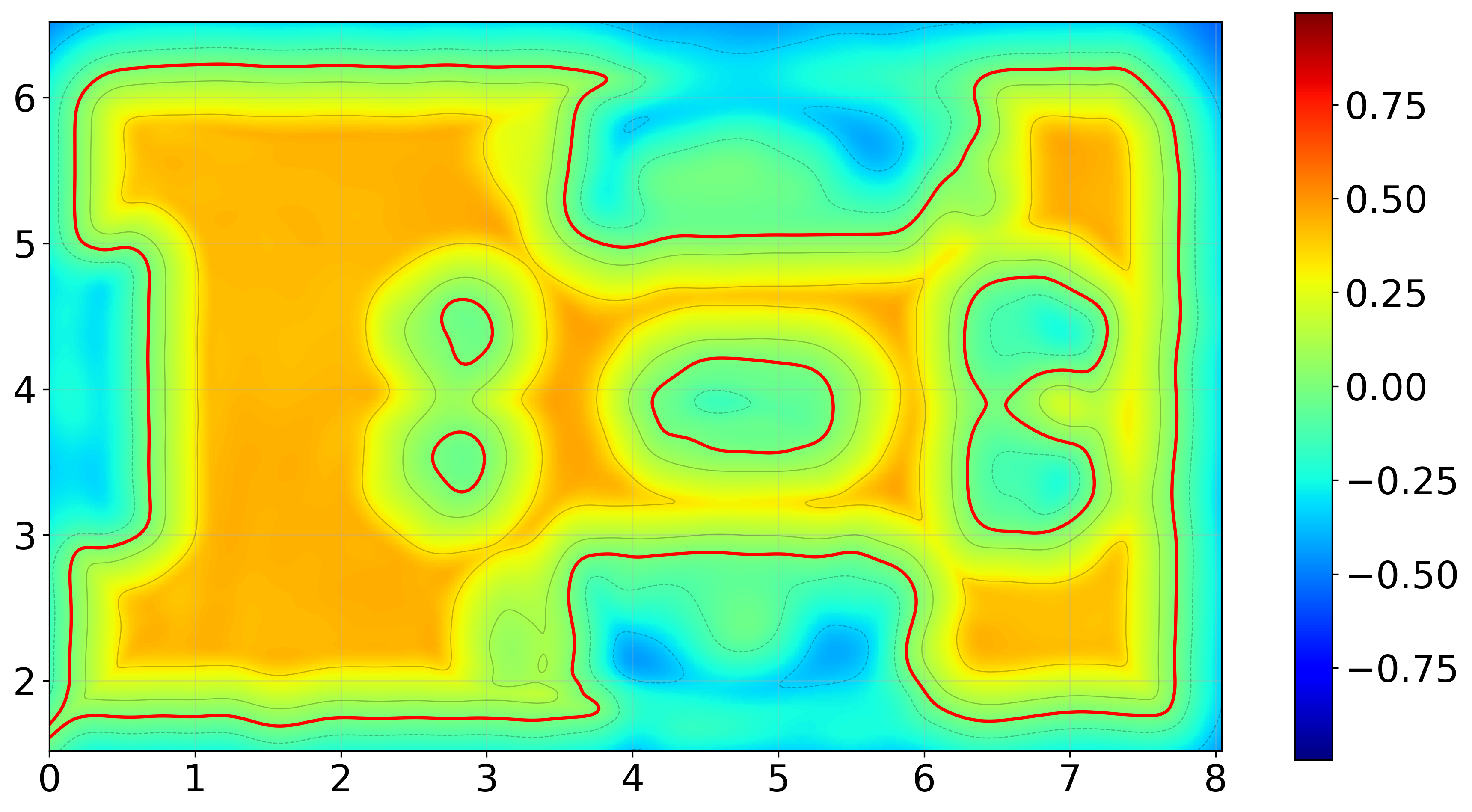}
            \caption{\footnotesize HIO-SDF \cite{hio-sdf_2024}}
        \end{subfigure}
        \begin{subfigure}[t]{0.16\linewidth}
            \includegraphics[width=\linewidth]{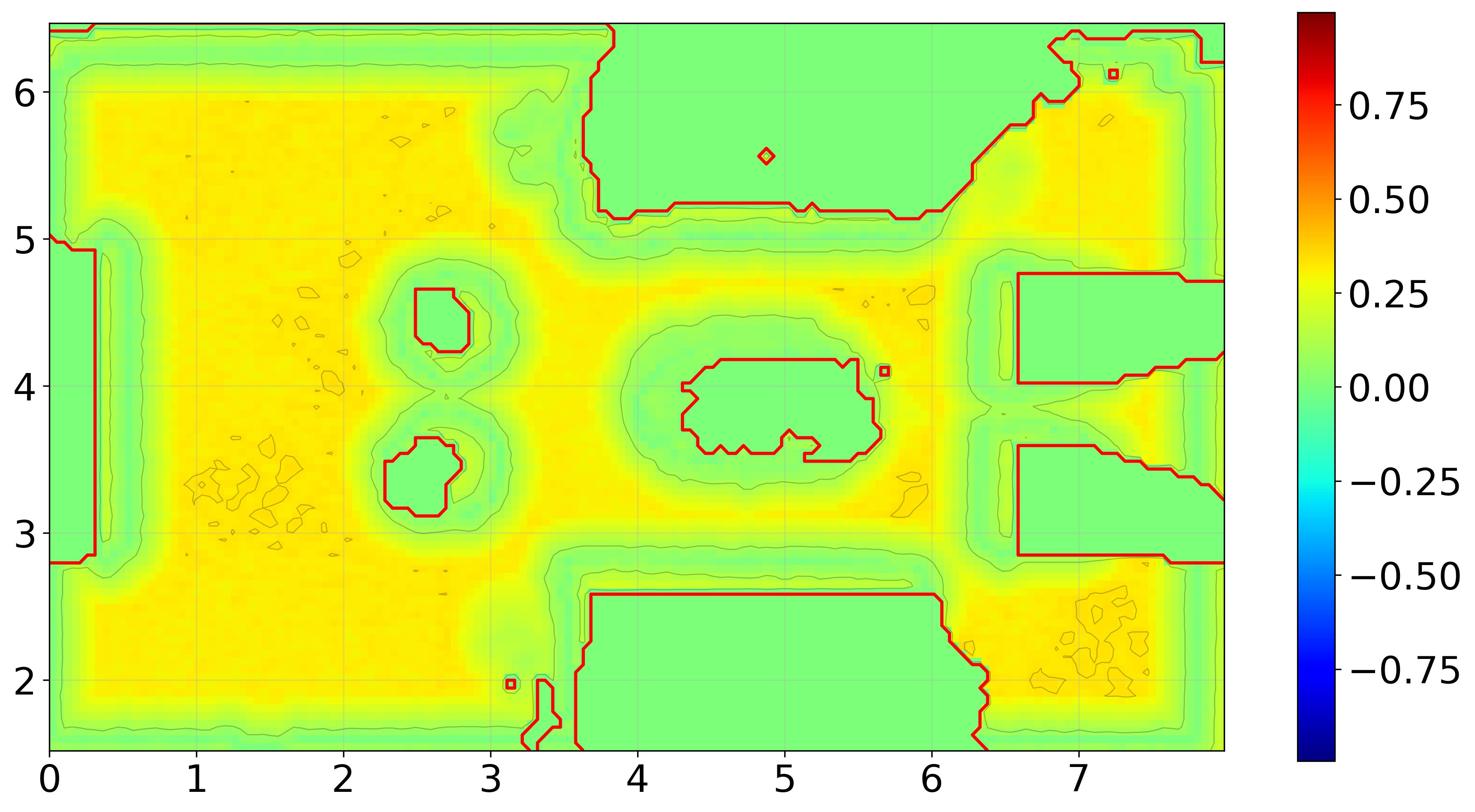}
            \caption{\footnotesize Voxblox \cite{oleynikova_voxblox_2017}}
        \end{subfigure}
    \end{center}
    \vspace{-2mm}
    \caption{\small Qualitative comparison of mesh reconstruction (top row) and z-plane slice of SDF reconstruction (bottom row) on Replica room 0 \cite{replica19arxiv}. \methodname reconstructs a mesh with the highest completion ratio and accurate SDF both near and far from the surface. \htwomapping and PIN-SLAM only learn truncated SDF. HIO-SDF learns an over smooth result. Voxblox significantly under-estimates the SDF.}
    \label{fig:qualitative}
    \vspace{-1em}
\end{figure*}

\begin{figure}[t]
    \centering
     \begin{subfigure}[t]{0.48\linewidth}
        \includegraphics[width=\linewidth]{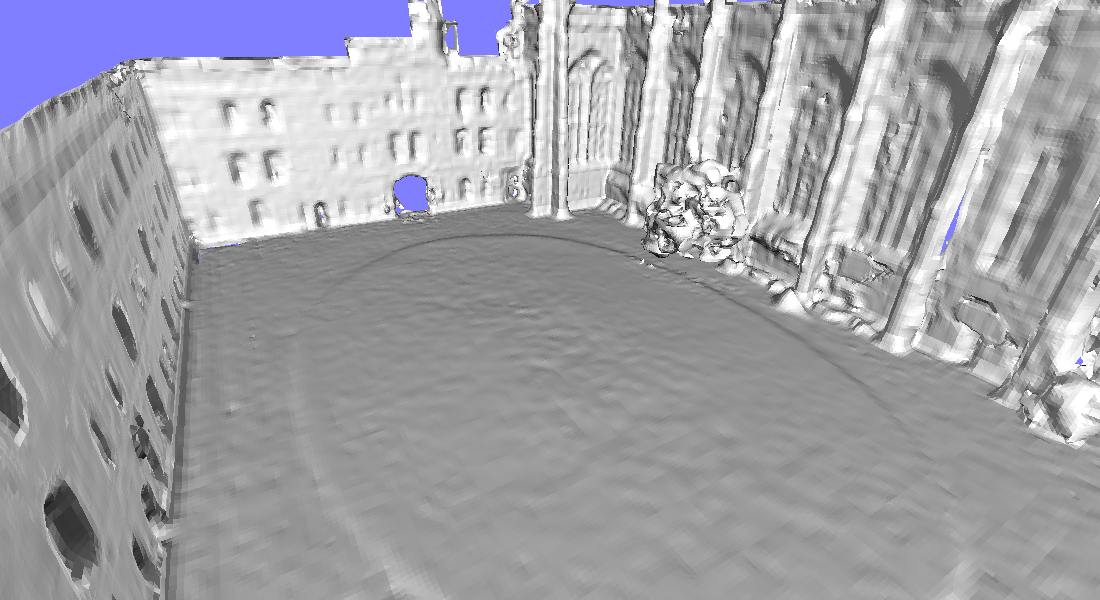}
        \caption{\footnotesize \methodname}
    \end{subfigure}%
    \hfill%
    \begin{subfigure}[t]{0.48\linewidth}
        \includegraphics[width=\linewidth]{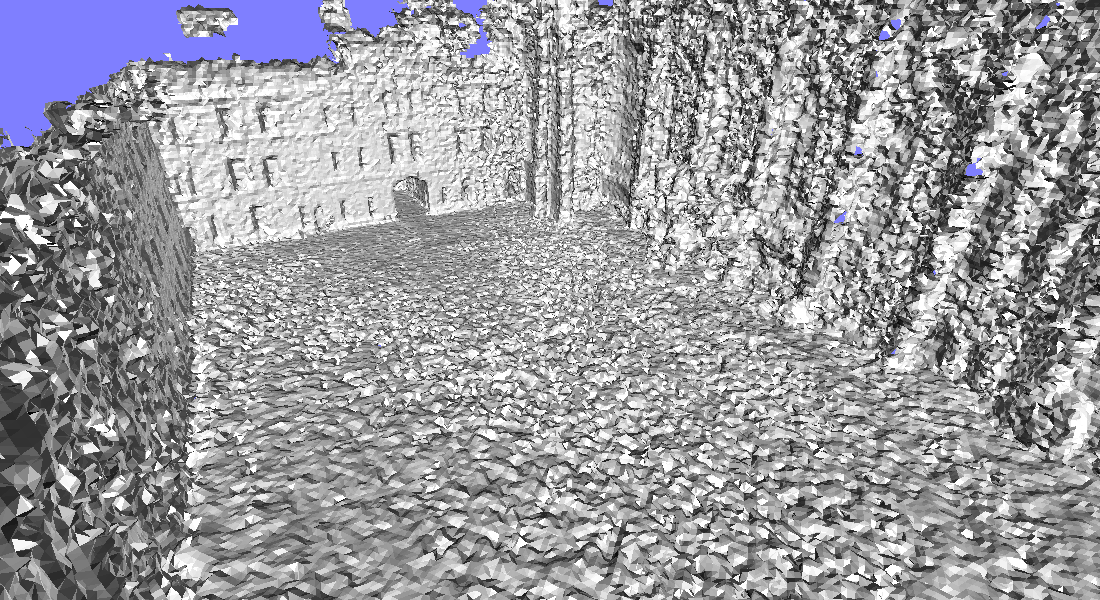}
        \caption{\footnotesize PIN-SLAM}
    \end{subfigure}%
\caption{\small Comparison of mesh reconstruction using \methodname versus PIN-SLAM \cite{pan_pin-slam_2024} on the Newer College dataset \cite{newercollege2021}.}
\label{fig:comp_mesh_newer_college}
\vspace{-3ex}
\end{figure}

\begin{figure}[t]
    \centering
    \begin{subfigure}[t]{0.24\linewidth}
        \includegraphics[width=\linewidth]{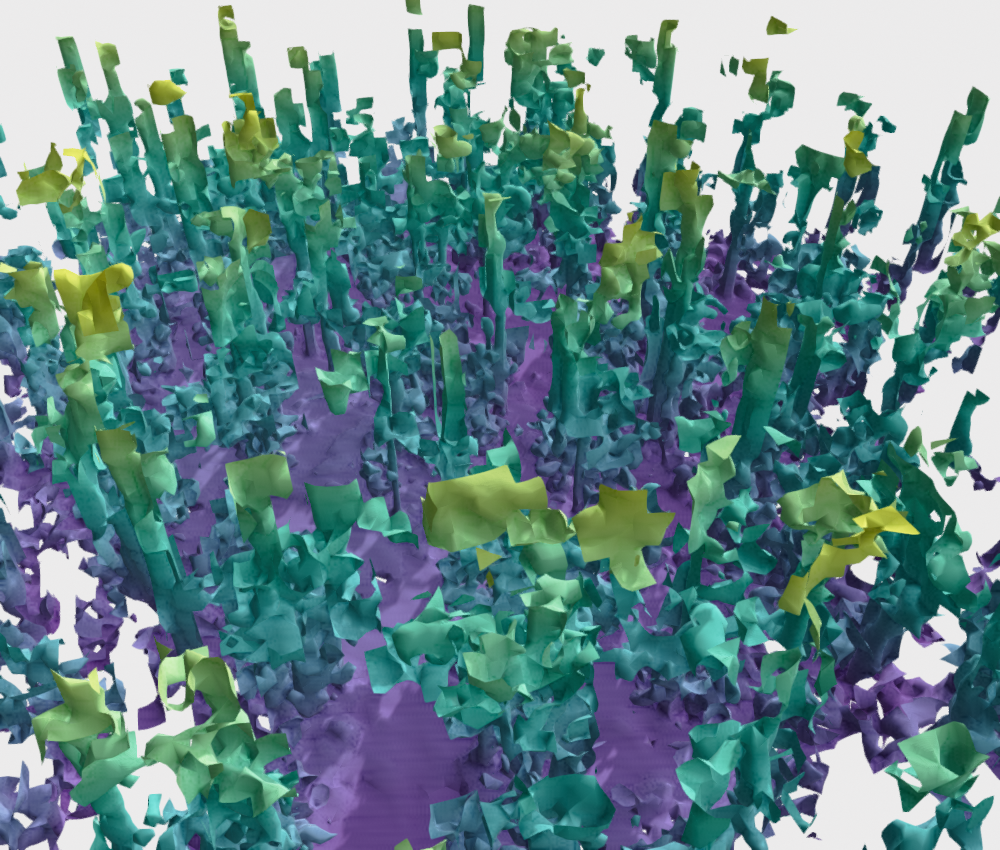}
        \caption{\footnotesize Forest}
    \end{subfigure}%
    \hfill%
    \begin{subfigure}[t]{0.24\linewidth}
        \includegraphics[width=\linewidth]{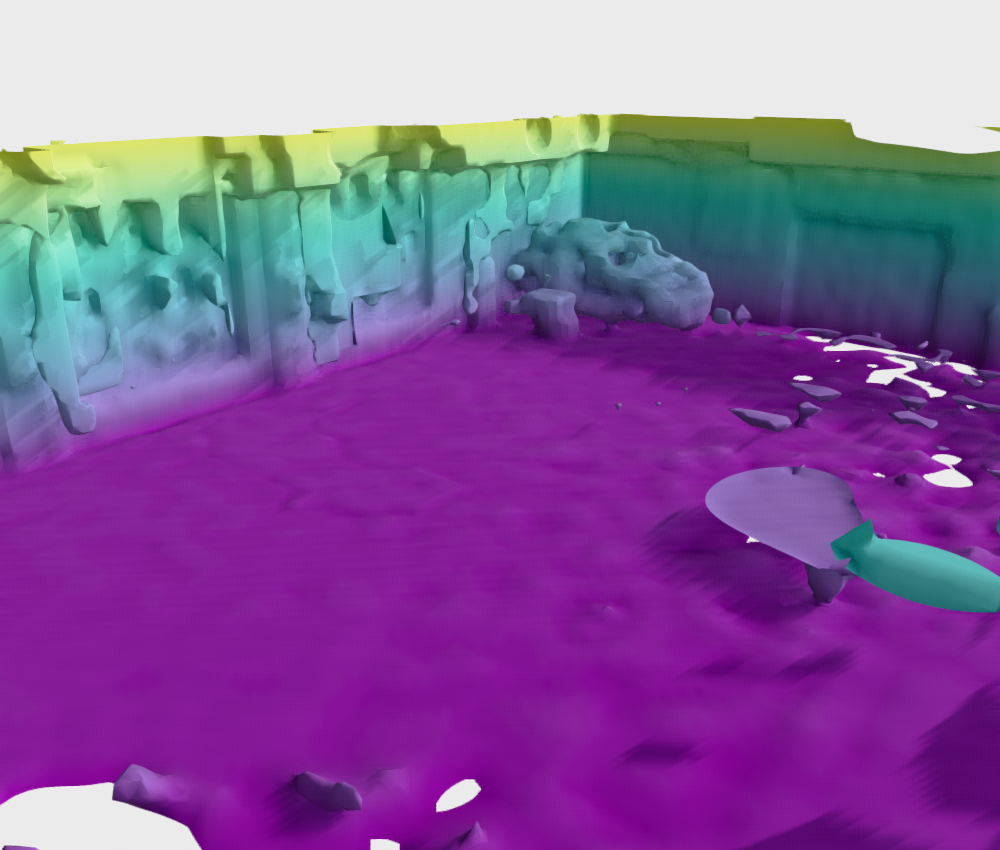}
        \caption{\footnotesize Garage}
    \end{subfigure}
    \begin{subfigure}[t]{0.24\linewidth}
        \includegraphics[width=\linewidth]{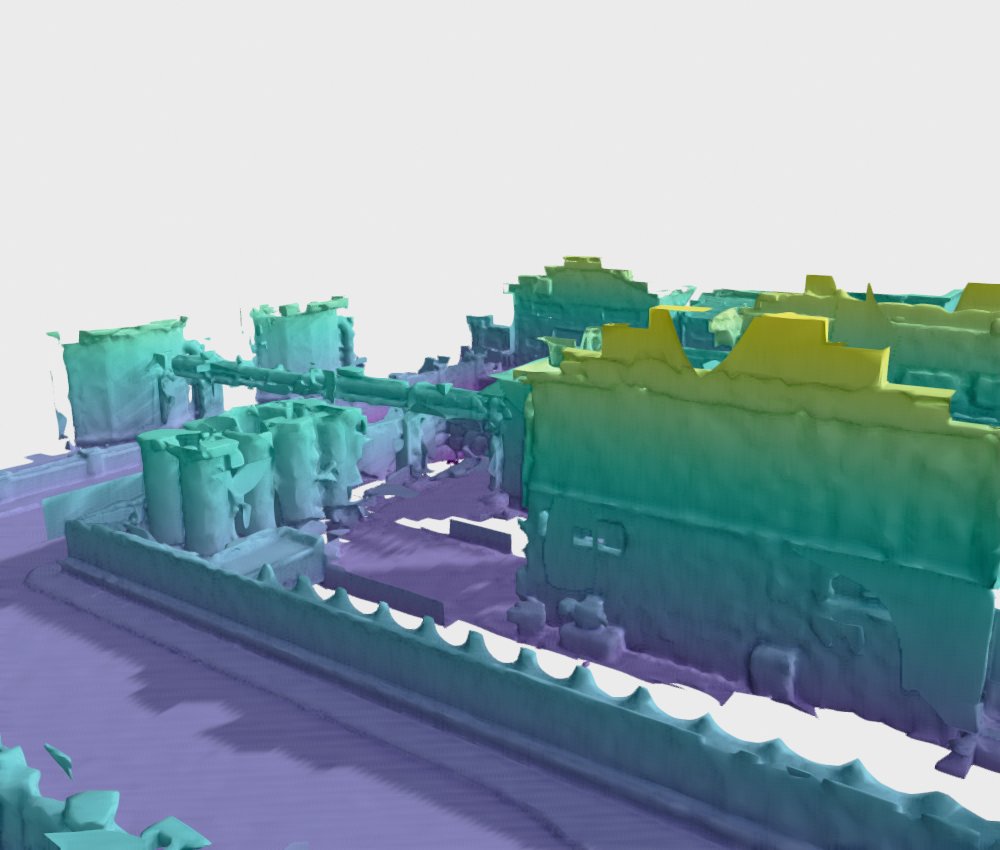}
        \caption{\footnotesize Industrial}
    \end{subfigure}%
    \hfill%
    \begin{subfigure}[t]{0.24\linewidth}
        \includegraphics[width=\linewidth]{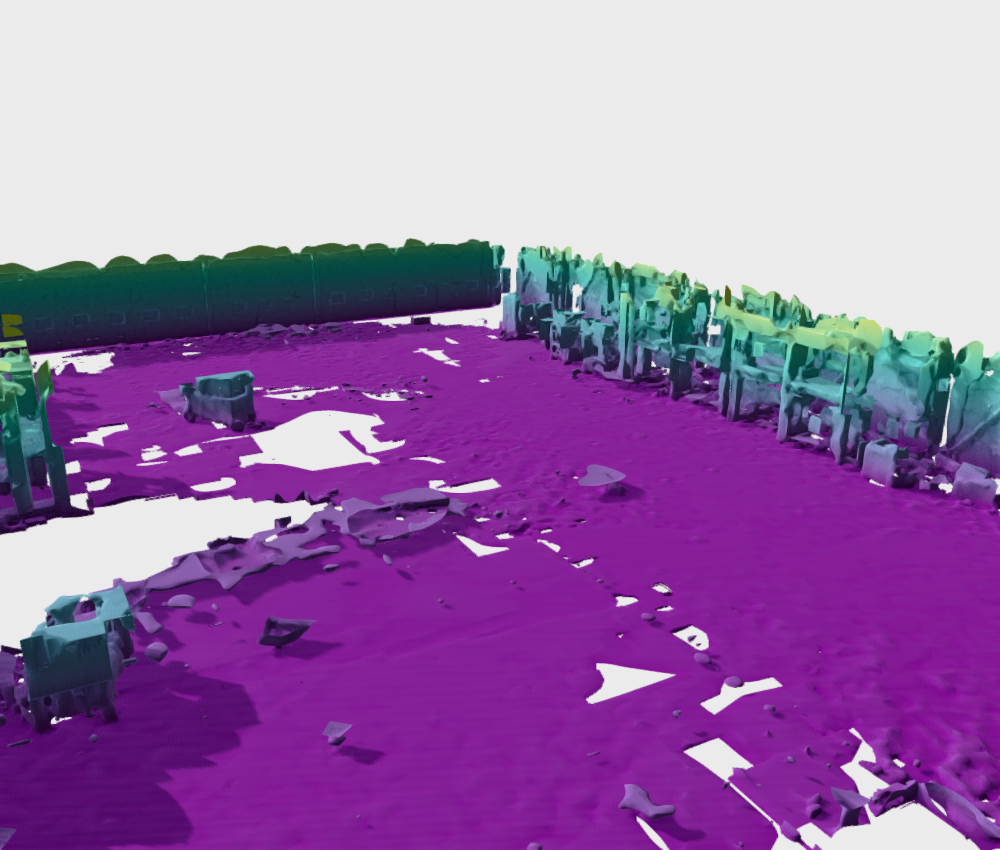}
        \caption{\footnotesize Warehouse}
    \end{subfigure}
\caption{\small Reconstructions by \methodname on four large outdoor scenes.}
\label{fig:outdoor_scenes}
\vspace{-3ex}
\end{figure}

Fig. \ref{fig:qualitative} shows mesh reconstruction results on the Replica room 0 scene. \htwomapping, PIN-SLAM, and \methodname generate complete and high-quality meshes. \htwomapping\ tends to produce smoother surfaces as it only allocates octree voxels near the surface.
The completeness of our reconstructions is better than \htwomapping. Compared with PIN-SLAM \cite{pan_pin-slam_2024}, which only predicts SDF values in regions close to the surface, our meshes exhibit smoother geometry with less noise. Compared to HIO-SDF \cite{hio-sdf_2024}, which also estimates continuous and differentiable SDFs, our reconstructions demonstrate substantially higher fidelity.
We also present the reconstructed meshes of the Newer College dataset by \methodname and PIN-SLAM in Fig. \ref{fig:comp_mesh_newer_college}, which shows that our method also outperforms the baselines on large scale real dataset.
We also show reconstructions on four diverse large-scale outdoor scenes in Fig. \ref{fig:outdoor_scenes}. \methodname runs in real-time on these scenes.

\subsubsection{SDF Reconstruction}
We visualize z-plane slices of the SDF predictions in Fig. \ref{fig:qualitative}. \htwomapping\ and PIN-SLAM estimate truncated SDF only. Although HIO-SDF can predict SDF values over the entire space, it struggles to precisely encode the surface as a zero-level set, which is crucial for robotic tasks where accurate perception of obstacle boundaries is required. In contrast, \methodname faithfully reconstructs the surface position and provides reliable SDF estimates in free space. The predictions of \methodname outside the scene boundaries are less accurate due to the lack of sensor observations. But this has little impact on applications where robots operate within the observed workspace.

\subsection{Quantitative Results}

We compute two sets of metrics: mesh metrics to evaluate the surface reconstruction quality and SDF metrics to evaluate the overall SDF predictions.

\subsubsection{Mesh Metrics}

We uniformly sample two point clouds, $\calP_\text{g.t.}$ and $\calP_\text{recon}$, of $200$k points each from the ground-truth mesh and from the reconstructed mesh, and report \emph{Chamfer-L1 distance}, \emph{F1 score}, \emph{precision}, \emph{recall}, \emph{completion}, \emph{completion ratio}, and \emph{accuracy}\cite{jiang_h2-mapping_2023,pan_pin-slam_2024}.

\begin{table*}[htbp]
    \centering
    \caption{\small Mesh reconstruction metrics on the Replica dataset \cite{replica19arxiv} and the Newer College dataset \cite{newercollege2021}.
    The best and second best results are bold and underlined, respectively.
    For the Replica dataset, $\delta=5$cm. For the Newer College dataset, $\delta=20$cm.
    }
    \label{table:mesh metrics}
    \scriptsize
    \begin{tabular}{l|l|ccccccccc}
    \hline
    Metric & Method & room 0 & room 1 & room 2 & office 0 & office 1 & office 2 & office 3 & office 4 & Newer College\\
    \hline
    \multirow{5}{*}{Completion [cm] $\downarrow$}
    & \methodname  &  \textbf{2.55} &  \textbf{1.71} &  \textbf{1.72} &  \textbf{1.52} &  \textbf{1.50} &  \textbf{2.41} &  \textbf{2.83} &  \textbf{2.76} & \textbf{10.66} \\
    & \htwomapping  &  \underline{3.05} & 2.51 & \underline{2.14} &  \underline{1.64} &  \underline{1.98} &  \underline{3.06} &  \underline{3.09} &  \underline{3.35} &21.94 \\
    & PIN-SLAM    &  3.46 &  2.90 &  2.89 &  1.96 &  2.58 &  3.54 &  3.34 &  3.77 & \underline{16.71}\\
    & HIO-SDF     &  4.50 &  3.63 &  3.51 &  3.26 &  3.87 &  4.09 &  4.50 &  4.17 & 72.86\\
    & Voxblox & 3.69 & \underline{2.16} & 2.31 & 1.71 & 2.07 & 3.52 & 4.37 & 4.38 & 21.30 \\
    \hline
    \multirow{5}{*}{Completion Ratio [$<\delta$]\% $\uparrow$}
    & \methodname  &  \textbf{93.32} &  \textbf{96.06} &  \textbf{95.66} &  \textbf{97.13} &  \textbf{96.51} &  \textbf{92.44} &  \textbf{89.52} &  \textbf{90.67} & \textbf{94.20} \\
    & \htwomapping  &  \underline{90.79} &  \underline{91.87} &  \underline{92.87} &  \underline{94.75} &  \underline{92.11} &  \underline{89.19} &  \underline{88.45} &  \underline{87.88} & 61.58 \\
    & PIN-SLAM    &  88.84 &  90.30 &  89.71 &  92.96 &  89.34 &  86.72 &  86.64 &  86.36 & \underline{72.83}\\
    & HIO-SDF     &  78.41 &  83.04 &  84.68 &  84.83 &  79.82 &  79.20 &  78.65 &  80.30 & 10.05\\
    & Voxblox  & 87.09 & 91.26 & 90.41 & 91.92 & 89.85 & 84.33 & 80.21 & 83.06 & 60.31 \\
    \hline
    \multirow{5}{*}{Recall [$<\delta$]\% $\uparrow$}
    & \methodname  &  \textbf{92.87} &  \textbf{95.84} &  \textbf{95.48} &  \textbf{96.81} &  \textbf{96.21} &  \textbf{92.26} &  \underline{88.49} &  \textbf{90.51} & \textbf{93.99} \\
    & \htwomapping  &  \underline{91.53} &  \underline{92.46} &  \underline{93.32} &  \underline{95.00} &  \underline{92.65} &  \underline{90.21} &  \textbf{89.56} &  \underline{89.14} & 57.96\\
    & PIN-SLAM    &  89.83 &  91.03 &  90.55 &  93.29 &  90.22 &  88.31 &  87.87 &  87.82 & \underline{70.40}\\
    & HIO-SDF     &  79.93 &  84.00 &  85.53 &  85.75 &  81.44 &  80.31 &  80.52 &  81.74 & 4.72\\
    & Voxblox  &  88.17 & 91.81 & 90.79 & 92.18 & 90.62 & 85.67 & 82.21 &  84.94 & 56.64 \\
    \hline
    \multirow{5}{*}{Precision [$<\delta$]\% $\uparrow$}
    & \methodname  &  87.04 &  90.77 &  91.70 &  86.90 &  88.69 &  90.14 &  80.59 &  89.00 & \textbf{90.69} \\
    & \htwomapping  &  \textbf{99.57} &  \textbf{99.78} &  \textbf{99.59} &  \textbf{99.65} &  \textbf{99.44} &  \textbf{99.66} &  \textbf{99.08} &  \textbf{99.52} & 52.97\\
    & PIN-SLAM      &  \underline{98.57} & \underline{98.39} &  \underline{98.59} &  \underline{97.95} &  \underline{98.40} &  \underline{98.06} &  \underline{96.83} &  \underline{98.39} & \underline{64.63}\\
    & HIO-SDF       &  85.86 &  89.04 &  90.54 &  91.24 &  88.55 &  84.83 &  88.24 &  88.15 & 4.46\\
    & Voxblox  & 96.18 & 97.94 & 94.56 & 95.31 & 98.13 & 93.69 &  91.44 & 95.55 & 51.84\\
    \hline
    \multirow{5}{*}{F1 Score [$<\delta$]\% $\uparrow$}
    & \methodname  &  89.86 &  93.24 &  93.55 &  91.59 &  92.29 &  91.19 &  84.36 &  89.75 & \textbf{92.31} \\
    & \htwomapping  &  \textbf{95.38} &  \textbf{95.98} &  \textbf{96.35} &  \textbf{97.27} &  \textbf{95.92} &  \textbf{94.70} &  \textbf{94.08} &  \textbf{94.05} & 55.35\\
    & PIN-SLAM    &  \underline{93.00} & 94.57 &  \underline{94.40} &  \underline{95.57} & 94.13 &  \underline{92.93} &  \underline{92.13} &  \underline{92.81} & \underline{67.39}\\
    & HIO-SDF     &  82.83 &  86.45 &  87.96 &  88.41 &  84.85 &  82.51 &  84.20 &  84.82 & 4.59\\
    & Voxblox  & 92.00 & \underline{94.77} & 92.64 & 93.72 & \underline{94.23} & 89.50 &  86.58 & 89.94 & 54.14 \\
    \hline
    \multirow{5}{*}{Chamfer-L1 Distance [cm] $\downarrow$}
    & \methodname  &  2.58 &  1.96 & \underline{1.85} &  2.07 & 1.85 &  \underline{2.35} &  3.26 &  \underline{2.60} & \textbf{9.36} \\
    & \htwomapping  &  \textbf{2.29} &  \underline{1.86} &  \textbf{1.73} &  1.83 & \underline{1.49} &  \textbf{2.24} &  \textbf{2.39} &  \textbf{2.46} & 28.4\\
    & PIN-SLAM      &  \underline{2.56} &  2.13 &  2.19 &  \underline{1.70} &  1.99 &  2.57 & \underline{2.60} &  2.73 & \underline{21.64}\\
    & HIO-SDF       &  3.79 &  3.24 &  3.14 &  2.93 &  3.33 &  3.79 &  3.90 &  3.60 & 422.29\\
    & Voxblox       &  2.72 & \textbf{1.65} & 1.90 & \textbf{1.47} & \textbf{1.44} & 2.60 & 3.26 & 3.02 & 23.37 \\
    \hline
    \multirow{5}{*}{Accuracy [cm] $\downarrow$}
    & \methodname  &  2.60 &  2.21 &  1.98 &  2.62 &  2.20 &  2.29 &  3.69 &  2.44 &  \textbf{8.07} \\
    & \htwomapping  &  \textbf{1.54} &  \underline{1.21} &  \textbf{1.33} &  \textbf{1.22} &  \underline{1.01} &  \textbf{1.43} &  \textbf{1.69} &  \textbf{1.58} & 34.86\\
    & PIN-SLAM    &  \underline{1.66} &  1.37 &  \underline{1.49} &  1.45 &  1.31 & \underline{1.61} & \underline{1.87} &  1.70 & 26.58\\
    & HIO-SDF     &  3.09 &  2.85 &  2.78 &  2.62 &  2.79 &  1.49 &  3.31 &  3.03 & 771.72\\
    & Voxblox  &  {1.74} &  \textbf{1.14} & 1.50 &  \underline{1.22} &  \textbf{0.81} & 1.68 & 2.16 & \underline{1.67} & \underline{25.44} \\
    \hline
    \end{tabular}
    \vspace{-1em}
\end{table*}

As shown in Table \ref{table:mesh metrics}, \methodname outperforms the baselines in recall, completion, and completion ratio. Our method is mostly the second best in the other mesh metrics. Since \htwomapping\ and PIN-SLAM are specially optimized for surface reconstruction, they perform slightly better in metrics like F1 score. However, their mesh has more holes, which can also improve metrics like precision. Voxblox performs mostly worse than other methods. Since HIO-SDF relies on a volumetric method like Voxfield~\cite{pan_voxfield_2022} to generate the SDF dataset, it also performs worse than \methodname.

\subsubsection{SDF Metrics}

We evaluate the SDF predictions of all methods using a regular grid of resolution $1.25$ cm ($20$ cm for Newer College).
For each scene, the ground-truth SDF $d$ of each grid center is computed from the ground-truth mesh (point cloud for Newer College).
In Table \ref{table:SDF metrics}, we report the \emph{mean absolute error} (MAE) of SDF, the \emph{angular MAE} of the SDF gradient, and the SDF \emph{valid ratio}, defined as the proportion of test positions where a method can predict SDF.
To examine the prediction quality in different regions, we categorize the grid points with $-0.1 \leq d \leq 0.2$ m as near surface, and the other points as far from the surface. \methodname marginally performs better than the baselines in all SDF metrics. HIO-SDF fails to train the network stably when its volumetric method does not provide good results. Since \htwomapping and PIN-SLAM are able to predict SDF only near the surface, their SDF valid ratios are extremely low, while HIO-SDF and \methodname both cover the whole scene.

\subsubsection{Runtime Metrics}

\begin{table*}[htbp]
    \centering
    \caption{\small SDF reconstruction metrics on the Replica dataset \cite{replica19arxiv} and the Newer College dataset \cite{newercollege2021}.}
    \label{table:SDF metrics}
    \scriptsize
    \begin{tabular}{ll|c|ccccccccc}
    \hline
    Metric & Region & Method & room 0 & room 1 & room 2 & office 0 & office 1 & office 2 & office 3 & office 4 & Newer College \\
    \hline
    \multirow{11}{*}{SDF MAE [cm]$\downarrow$}
    & \multirow{3}{*}{All}
     & \methodname   &  \textbf{2.15} & \textbf{2.13} & \textbf{2.26} &  \textbf{1.93} &  \textbf{2.16} &  \textbf{2.32} &  \textbf{2.65} &  \textbf{2.36} & \textbf{56.00} \\
    && HIO-SDF &  3.27 & 2.79 &  39.98 &  \underline{2.60} &  \underline{2.72} &  \underline{3.34} &  \underline{3.40} & 3.39 & 301.56 \\
    && Voxblox & \underline{3.13} & \underline{2.54} & \underline{2.73} & 3.02 &  2.73 & 3.47 & 3.74 & \underline{3.08} & \underline{62.92} \\
    \hhline{|~|-|-|-|-|-|-|-|-|-|-|-|}
    & \multirow{5}{*}{Near}
     & \methodname   &  \textbf{1.67} &  \textbf{1.26} &  \textbf{1.82} &  \textbf{1.30} &  \textbf{1.43} &  \textbf{1.84} &  \textbf{2.15} &  \textbf{2.11} & \underline{16.96} \\
    && HIO-SDF         & 3.45 & 2.90 &  28.32 &  2.57 & 3.00 & 3.36 & 3.52 & 3.62 & 58.34 \\
    && \htwomapping    &  6.13 &  6.01 &  5.54 &  5.88 &  5.75 &  5.81 &  5.99 &  6.19 & \textbf{9.60}\\
    && PIN-SLAM        &  4.65 &  8.10 &  8.06 &  8.06 &  8.15 &  8.11 &  8.13 &  8.07 & 22.64\\
    && Voxblox         & \underline{2.78} &  \underline{2.15} &  \underline{2.23} &  \underline{2.20} &  \underline{2.17} & \underline{2.85} & \underline{3.41} &  \underline{2.96} & 19.63 \\
    \hhline{|~|-|-|-|-|-|-|-|-|-|-|-|}
    & \multirow{3}{*}{Far}
     & \methodname & \textbf{2.39} &  \textbf{2.66} &  \textbf{2.49} &  \textbf{2.33} &  \underline{2.70} &  \textbf{2.58} &  \textbf{2.98} & \textbf{2.49} & \textbf{59.62} \\
    && HIO-SDF  &  \underline{3.12} &  \underline{2.69} &  50.22 &  \underline{2.62} &  \textbf{2.39} &  \underline{3.33} &  \underline{3.30} & 3.20 & 324.12 \\
    && Voxblox  & 3.37 & 2.85 & \underline{3.08} & 3.71 & 3.27 &  3.92 &  4.00 & \underline{3.16} & \underline{67.32} \\
    \hline
    \multirow{11}{*}{Grad. MAE [rad]$\downarrow$}
    & \multirow{3}{*}{All}
     & \methodname   &  \textbf{0.153} &  \textbf{0.157} &  \textbf{0.160} &  \textbf{0.160} &  \textbf{0.174} &  \textbf{0.162} &  \textbf{0.158} &  \textbf{0.154} & \textbf{0.184} \\
    && HIO-SDF  & 0.924 &  1.315 &  1.534 &  1.272 &  1.101 &  1.340 &  1.326 &  1.256 & 1.529\\
    && Voxblox  & \underline{0.230} & \underline{0.195} & \underline{0.222} & \underline{0.251} & \underline{0.190} & \underline{0.271} & \underline{0.279} & \underline{0.232} & \underline{0.777} \\
    \hhline{|~|-|-|-|-|-|-|-|-|-|-|-|}
    & \multirow{5}{*}{Near}
     & \methodname & \textbf{0.120} &  \textbf{0.116} &  \textbf{0.117} &  \textbf{0.128} &  \textbf{0.129} &  \textbf{0.117} &  \textbf{0.130} &  \textbf{0.117} & \textbf{0.156} \\
    && \htwomapping  &  1.076 &  1.095 &  1.043 &  1.095 &  1.081 &  1.085 &  1.081 &  1.096 & \underline{1.379}\\
    && PIN-SLAM      &  0.870 &  1.566 &  1.571 &  1.569 &  1.530 &  1.578 &  1.550 &  1.572 & 1.533\\
    && HIO-SDF       &  0.975 &  1.301 &  1.560 &  1.281 &  1.115 &  1.324 &  1.331 &  1.257 & 1.565\\
    && Voxblox       & \underline{0.282} & \underline{0.183} & \underline{0.204} & \underline{0.259} & \underline{0.179} & \underline{0.301} & \underline{0.336} & \underline{0.299} & 1.435 \\
    \hhline{|~|-|-|-|-|-|-|-|-|-|-|-|}
    & \multirow{3}{*}{Far}
     & \methodname &  \textbf{0.165} &  \textbf{0.177} &  \textbf{0.176} &  \textbf{0.175} &  \underline{0.199} &  \textbf{0.180} &  \textbf{0.171} &  \textbf{0.167} & \textbf{0.198} \\
    && HIO-SDF  &  0.884 & 1.328 &  1.512 & 1.262 & 1.084 & 1.356 & 1.325 &  1.255 & 1.526 \\
    && Voxblox  & \underline{0.208} & \underline{0.201} & \underline{0.230} & \underline{0.246} & \textbf{0.198} & \underline{0.257} & \underline{0.256} & \underline{0.205} & \underline{0.720} \\
    \hline
    \multicolumn{2}{l|}
    {\multirow{2}{*}{SDF Valid Ratio [\%]$\uparrow$}}
    & \htwomapping  &  \underline{16.05} &  \textbf{16.62} &  \textbf{18.15} &  \textbf{17.46} &  \textbf{20.30} &  \textbf{16.84} &  \textbf{17.05} &  \textbf{15.43} & \underline{28.22}\\
    && PIN-SLAM    &  \textbf{25.01} &  \underline{6.47} &  \underline{5.82} &  \underline{7.86} &  \underline{8.63} &  \underline{7.09} &  \underline{7.19} &  \underline{5.67} & \textbf{47.19} \\
    \hline
    \end{tabular}
    \vspace{-3ex}
    \end{table*}
\def\ourfps{13.96\xspace}
\def\ournewerfps{14.28\xspace}
\begin{table}[t]
\centering
\caption{\small Runtime and peak memory on Replica room 0 \cite{replica19arxiv} and Newer College \cite{newercollege2021}. \htwomapping \cite{jiang_h2-mapping_2023} is tested on Newer College by converting each LiDAR scan into 4 depth images since its implementation lacks LiDAR support.}
\label{table:comp_runtime}
\footnotesize
\resizebox{\linewidth}{!}{%
\begin{tabular}{ll|ccccc}
\hline
& & \methodname & \htwomapping & PIN-SLAM & HIO-SDF & Voxblox \\
\hline
\multirow{3}{*}{\rotatebox{90}{\shortstack{Replica\\room 0}}}
& FPS $\uparrow$        & \textbf{\ourfps} & \underline{12.36} & 8.43 & 1.99 & 0.87 \\
& CPU [GB] $\downarrow$ & 3.81 & 2.46 & \underline{2.23} & 10.33 & \textbf{1.19} \\
& GPU [GB] $\downarrow$ & \textbf{1.33} & 13.96 & 6.12 & 2.14 & N/A \\
\hline
\multirow{3}{*}{\rotatebox{90}{\shortstack{Newer\\College}}}
& FPS $\uparrow$        & \textbf{\ournewerfps} & \underline{$9.03^*$} & 8.95 & 3.80 & 0.20 \\
& CPU [GB] $\downarrow$ & 3.10 & \underline{2.37} & 5.98 & 7.08 & \textbf{2.04} \\
& GPU [GB] $\downarrow$ & \textbf{1.59} & 9.90 & 7.56 & \underline{2.63} & N/A \\
\hline
\end{tabular}}
\vspace{-1em}
\end{table}

We measure the timing and peak memory of each method on Intel 14900K CPU and NVIDIA RTX 3090 GPU. As shown in Table \ref{table:comp_runtime}, \methodname is the fastest on both the small-scale Replica room 0 and the large-scale Newer College sequence, running at \ourfps and \ournewerfps fps, respectively. Voxblox is known for its fast TSDF estimation but its ESDF estimation is significantly slower due to the integration. \methodname uses the least GPU memory among all methods (1.33 GB on Replica and 1.59 GB on Newer College). Its CPU memory is higher than the lightest baselines because the semi-sparse octree allocates additional vertices but stays far below HIO-SDF and modest even on the large Newer College scene. This efficiency is also confirmed by deployment on an NVIDIA Jetson Orin with 16 GB RAM, where \methodname runs at 7 fps.

\subsection{Ablation Study}

To assess each component, we train four variants: without the neural residual, with a regular sparse octree, without gradient augmentation, and without the projection loss.
As shown in Table \ref{table:ablation_study}, the neural residual improves both mesh reconstruction quality and SDF prediction accuracy. With a sparse octree, our method performs worse due to discontinuities. Without gradient augmentation, the SDF priors become worse, leading to worse performance. Removing the projection loss significantly worsens metrics, as it provides important guidance on the SDF scale and gradient direction.
Finally, the numbers of semi-sparse layers trades off memory for accuracy. On Replica room 0, making all layers semi-sparse nearly doubles the allocated octants from $21418$ to $41073$ ($+92\%$) and the vertices from $37130$ to $51174$ ($+38\%$), without significant accuracy gain, confirming that only the coarser layers need to be semi-sparse.

\begin{table}[t]
\centering
\caption{\small Ablation study results on Replica room 1 \cite{replica19arxiv}.}
\label{table:ablation_study}
\scriptsize
\begin{tabular}{ll|ccccc}
\hline
\multirow{2}{*}{Metric} & & \multirow{2}{*}{\methodname} & Prior & Sparse & w/o Grad.  & w/o Proj. \\
                        & &                                & Only  & Octree & Aug.       & Loss \\
\hline
\multicolumn{2}{l|}{F1 Score \%$\uparrow$}& \textbf{93.33} & \underline{93.07} & 93.01 & 92.09 & 89.22 \\
\multicolumn{2}{l|}{Precision \%$\uparrow$}& \textbf{90.98} & \underline{90.48} & 90.45 & 88.61 & 84.04 \\
\multicolumn{2}{l|}{Recall \%$\uparrow$}& 95.79 & \underline{95.82} & 95.71 & \textbf{95.86} & 95.07 \\
\multicolumn{2}{l|}{Completion Ratio \%$\uparrow$}& 96.00 & \underline{96.05} & 95.95 & \textbf{96.17} & 95.65 \\
\multicolumn{2}{l|}{Completion [cm]$\downarrow$}& \textbf{1.705} & 1.709 & 1.717 & \underline{1.709} & 1.775 \\
\multicolumn{2}{l|}{Accuracy [cm]$\downarrow$} & \textbf{2.201} & 2.234 & \underline{2.228} & 2.441 & 3.045 \\
\multicolumn{2}{l|}{Chamfer Distance $\downarrow$}& \textbf{1.953} & \underline{1.971} & 1.972 & 2.075 & 2.410 \\
\hline
{SDF} & All & \textbf{1.862} & 2.185  & \underline{2.090} & 2.208 & 18.688 \\
{MAE} & Near & \textbf{1.203} & 1.320 & \underline{1.278} & 1.612 & 2.303 \\
{[cm]$\downarrow$}& Far & \textbf{2.468} & 2.981 & 2.837 & \underline{2.757} & 33.756 \\
\hline
Grad. & All & \textbf{0.1527} & \underline{0.1529} & 0.1626 & 0.1661 & 0.1713 \\
MAE   & Near & \underline{0.1166} & \textbf{0.1107} & 0.1229 & 0.1298 & 0.1328 \\
{[rad]$\downarrow$}& Far & \textbf{0.1770} & \underline{0.1823} & 0.1895 & 0.1890 & 0.2519 \\
\hline
\end{tabular}
\vspace{-4mm}
\end{table}

\section{Conclusion}
\label{sec:conclusion}

This paper developed \methodname, an online hybrid method for globally accurate SDF estimation from streaming point cloud data at large scales. Our method combines an explicit octree prior with an implicit neural residual. Our evaluation shows that this combination achieves more accurate and efficient SDF reconstruction than state-of-the-art methods. Future work will focus on utilizing \methodname in robot localization, navigation, and manipulation.

\balance
{\small
\bibliographystyle{IEEEtran}
\bibliography{bib/main.bib}
}

\end{document}